\documentclass[journal,10pt]{IEEEtran}
\usepackage{booktabs}
\usepackage{times}
\usepackage{booktabs}
\usepackage{mathtools}
\usepackage{float}
\usepackage{algorithm,algorithmic}
\usepackage{amsmath}
\usepackage{amssymb}
\usepackage{color}
\usepackage{graphicx}
\usepackage{url}
\usepackage{multirow}
\usepackage{hyperref}
\usepackage{cite}

\usepackage{amsthm}

\allowdisplaybreaks

\newtheoremstyle{theoremdd}
	{\topsep}
	{\topsep}
	{\itshape}
	{0pt}
	{\bfseries}
	{. }
	{ }
	{\thmname{#1}\thmnumber{ #2}\thmnote{ (#3)}}

\theoremstyle{theoremdd}

\newcommand{\expect}{\operatorname{E}\expectarg}
\DeclarePairedDelimiterX{\expectarg}[1]{[}{]}{%
	\ifnum\currentgrouptype=16 \else\begingroup\fi
	\activatebar#1
	\ifnum\currentgrouptype=16 \else\endgroup\fi
}

\newcommand{\innermid}{\nonscript\;\delimsize\vert\nonscript\;}
\newcommand{\activatebar}{%
	\begingroup\lccode`\~=`\|
	\lowercase{\endgroup\let~}\innermid 
	\mathcode`|=\string"8000
}
\newcommand{\bs}{\boldsymbol}
\DeclareMathOperator*{\argmin}{arg\,min}
\DeclareMathOperator*{\argmax}{arg\,max}
\newcommand{\mr}{\mathrm}
\newcommand{\typwt}{\tau_{y,{p}_{\bs{w}}}(t)}

\newtheorem{assumption}{Assumption}
\newtheorem{lemma}{Lemma}
\newtheorem{remark}{Remark}
\newtheorem{theorem}{Theorem}

\begin{document}
	
    \title{Exploiting Relevance for Online Decision-Making in High-Dimensions}

	\author{Eralp Tur\u{g}ay,
		Cem Bulucu, 
		and Cem Tekin,~\IEEEmembership{Senior Member,~IEEE}
		\thanks{\textcopyright 2020 IEEE. Personal use of this material is permitted. Permission from IEEE must be obtained for all other uses, in any current or future media, including reprinting/republishing this material for advertising or promotional purposes, creating new collective works, for resale or redistribution to servers or lists, or reuse of any copyrighted component of this work in other works}
		\thanks{This work was supported in part by the Scientific and Technological Research Council of Turkey (TUBITAK) under Grants 116E229 and 215E342. \textit{(Eralp Tur\u{g}ay and Cem Bulucu contributed equally to this work.)}}
		\thanks{The authors are with the Department of Electrical and Electronics Engineering, Bilkent University, Ankara, Turkey, 06800 (e-mail: turgay@ee.bilkent.edu.tr; bulucu@ee.bilkent.edu.tr;  cemtekin@ee.bilkent.edu.tr).} 
		}

	\maketitle

\begin{abstract}
Many sequential decision-making tasks require choosing at each decision step the right action out of the vast set of possibilities by extracting actionable intelligence from high-dimensional data streams. Most of the times, the high-dimensionality of actions and data makes learning of the optimal actions by traditional learning methods impracticable. In this work, we investigate how to discover and leverage sparsity in actions and data to enable fast learning. As our learning model, we consider a structured contextual multi-armed bandit (CMAB) with high-dimensional arm (action) and context (data) sets, where the rewards depend only on a few relevant dimensions of the joint context-arm set, possibly in a non-linear way. 
We depart from the prior work by assuming a high-dimensional, continuum set of arms, and allow relevant context dimensions to vary for each arm. We propose a new online learning algorithm called {\em CMAB with Relevance Learning} (CMAB-RL) and prove that its time-averaged regret asymptotically goes to zero when the expected reward varies smoothly in contexts and arms. CMAB-RL enjoys a substantially improved regret bound compared to classical CMAB algorithms whose regrets depend on the number of dimensions $d_x$ and $d_a$ of the context and arm sets. Importantly, we show that when the learner has prior knowledge on sparsity, given in terms of upper bounds $\overline{d}_x$ and $\overline{d}_a$ on the number of relevant context and arm dimensions, then CMAB-RL achieves $\tilde{O}(T^{1 - 1 /(2 + 2\overline{d}_x  + \overline{d}_a)})$ regret. 
Finally, we illustrate how CMAB algorithms can be used for optimal personalized blood glucose control in type 1 diabetes mellitus patients, and show that CMAB-RL outperforms other contextual MAB algorithms in this task, where the contexts represent multimodal physiological data streams obtained from sensor readings and the arms represent bolus insulin doses that are appropriate for injection. 
\end{abstract}

\begin{IEEEkeywords}
    Online learning, contextual multi-armed bandit, regret bounds, dimensionality reduction, personalized medicine.
\end{IEEEkeywords}

%
\IEEEpeerreviewmaketitle

\section{Introduction}

\IEEEPARstart{A}{I}-enabled technologies are becoming ubiquitous for many applications that involve repeated decision-making under uncertainty. Delivering personalized medicine for treatment of complex diseases \cite{yoon2016discovery}, discovering and recommending interesting articles for a particular user from huge corpora of documents \cite{huang2018rating,tekin2018multi} and optimizing hyper-parameters of deep learning architectures given a particular dataset \cite{li2018hyperband} all require context-driven learning of optimal decisions over huge action sets. 
As the dimensionality of the contexts and the actions grow, learning the optimal decision for each context becomes a formidable task since what has been learned in the past cannot be used to accurately estimate the action rewards for the current context.
Nevertheless, in many high-dimensional settings, only a subset of context and action dimensions affect the reward. For instance, in controlling the blood glucose of type 1 diabetes mellitus (T1DM) patients, data analysis highlights that future blood glucose of a patient only depends on blood glucose before the treatment, dose of the treatment and carbohydrate intake, whilst the affect of other physiological and environmental variables on blood glucose are found to be negligible \cite{zhu18deept1dm, midroni18xgboostt1dm}. Similarly, when training deep neural networks, it is observed that in general not only a small subset of hyperparameters can be considered relevant, but also the content of relevant subset of hyperparameters differs from one task to another \cite{bergstra2012random}.

In this paper, we model online decision-making in high-dimensions as a multi-armed bandit (MAB) \cite{lai1,auer2002finite}. MABs have successfully modeled a wide set of applications that involve sequential decision-making under uncertainty ranging from dynamic spectrum sharing \cite{gai2010learning,liu2010distributed,cohen2014restless} to medical diagnosis \cite{tekin2016confidence}. Specifically, we formalize the problem as a contextual MAB (CMAB) \cite{slivkins2014contextual}, where the learner observes a $d_x$-dimensional context from a context set ${\cal X}$ at the beginning of each round before selecting a $d_a$-dimensional action (arm) from an arm set ${\cal A}$.\footnote{In general, ${\cal X}$ and ${\cal A}$ have uncountably many elements.} This generalizes the MAB model and allows the arms' reward distributions depend on the context. The goal of the learner in this setting is to compete with an oracle that selects at each round the arm with the highest expected reward for the current context. The cumulative loss of the learner with respect to this oracle is called {\em the regret}, thereby minimizing the regret is equivalent to maximizing the cumulative expected reward. The learner's time-averaged expected reward will approach to that of the oracle as long as it can keep its regret sublinearly growing over time. Being able to capture intricacies of data-driven decision-making, CMAB algorithms have been successfully used in recommender systems \cite{li2010contextual}, personalized medicine \cite{deng2014budgeted} and cognitive communications \cite{asadi2018fml}.

Since the cardinalities of ${\cal X}$ and ${\cal A}$ are very large, further assumptions on the problem structure are required to obtain sublinear in time regret. In this paper, we consider a variant of CMAB with similarity information \cite{slivkins2014contextual}, where the reward from a context-arm pair comes from a fixed distribution, expected rewards vary smoothly in contexts and arms, and no stochastic assumptions are made on how contexts arrive over time.\footnote{Analysis holds for any fixed sequence of contexts.} In this setting, dimensionality of the context and arm sets play a key role on the performance of learning algorithms \cite{lu2010contextual}. In the worst-case, the regret has exponential dependence on $d_x$ and $d_a$, and thus, grows almost linearly in time in high-dimensional problems.

This motivates us to develop a new CMAB model and algorithm that address the learning challenges arising from high-dimensional context and arm sets. As discussed in the preceding paragraphs, in many applications of the CMAB, although the contexts and arms are high-dimensional, the most relevant information is embedded into a small number of relevant dimensions. Therefore, we consider a CMAB problem with similarity information where the expected reward only depends on relevant subcomponents of the arms and contexts. While the relevant subcomponent of the arms is fixed, the relevant subcomponent of the contexts can be different for each arm. For instance, in personalized treatment assignment, each arm can represent a drug cocktail and each component of an arm may correspond to the dose of a particular drug. Then, the relevance information tells that the outcome of the treatment only depends on a subset of relevant drugs in the cocktail and a subset of contexts of the patient (e.g., physiological data, genomic data) that are relevant to the drug cocktail. Minimizing the regret in this problem is extremely challenging since the learner knows neither the reward distributions nor what is relevant beforehand. All of these need to be learned online by only using the observed contexts, the selected arms and the random rewards observed from the selected arms in the past. 

In this paper, we solve the problem described above by only assuming that the learner knows upper bounds $\overline{d}_x$ and $\overline{d}_a$ on the number of relevant context and arm dimensions. Essentially, we propose a new algorithm called CMAB with Relevance Learning (CMAB-RL) that learns the relevant context and arm dimensions to achieve $\tilde{O}(T^{1 - 1 /(2 + 2\overline{d}_x  + \overline{d}_a)})$ regret, while on the other hand, CMAB algorithms that do not learn the relevance achieve $\tilde{O}(T^{1 - 1 /(2 + {d}_x  + {d}_a)})$ regret in the worst-case \cite{lu2010contextual}. This implies that CMAB-RL has a better regret bound than these algorithms in terms of its dependence on time as long as $2\overline{d}_x < d_x$ is satisfied, and significantly improves over the prior work for sparse MAB problems, where $\overline{d}_x << d_x$ and/or $\overline{d}_a << d_a$.

The most closely related work to ours is \cite{tekin2015releaf}, which considers a CMAB problem with finite number of arms, where the relevant context dimensions may vary from arm to arm. Provided with the same upper bound on the number of relevant context dimensions, the algorithm RELEAF in \cite{tekin2015releaf} is shown to achieve $\tilde{O}(T^{g(\bar{d}_x)})$ regret, where $g(\bar{d}_x) = (2+2\bar{d}_x + \sqrt{4{\bar{d}_x}^2 + 16\bar{d}_x + 12}) / (4+2\bar{d}_x + \sqrt{4{\bar{d}_x}^2 + 16\bar{d}_x + 12})$. However, the setting in \cite{tekin2015releaf} is quite different from ours, since the authors assume that reward feedback is costly, and thus, needs to be acquired only when there is a need to explore. Therefore, their algorithm achieves a worse regret bound than CMAB-RL (the regret of CMAB-RL for this setting is $\tilde{O}(T^{(1+2\bar{d}_x)/(2+2\bar{d}_x )})$), because it needs to rely on control functions to either perform exploration or exploitation in each round, while CMAB-RL does not explicitly separate these two. Moreover, our formulation allows us to deal with high-dimensional and continuum sets of arms, which can be used in representing action sets for drug dosage, online auctions \cite{blum2004online}, routing \cite{bansal2003online}, web-based recommendations \cite{song2014online} and web page content optimization \cite{hill2017efficient}.

In the core of CMAB-RL reside two new methods to identify and exploit relevance. The first one generates a collection of partitions of the context and arm sets formed by low-dimensional subsets of context and arm dimensions. This allows CMAB-RL to estimate rewards of context-arm pairs for only certain subsets of context and arm dimensions, thereby mitigating estimation errors caused by sparsity of similar samples that emerge from high-dimensionality. The second one identifies for each arm the candidate relevant tuples of context dimensions by comparing the variation of the sample mean rewards with confidence intervals constructed using selection statistics of related context-arm pairs. After identifying the candidate relevant tuples, CMAB-RL chooses the tuple with the minimum variation for each arm. Then, it uses the selected tuples to form reward estimates, and uses the principle of optimism in the face of uncertainty to minimize its regret. 

Apart from the regret bounds, we also show the superiority of CMAB-RL as compared to other learning methods via extensive simulations on synthetic and real-world datasets. We model optimal personalized blood glucose control problem in T1DM patients for the first time (to the best of our knowledge) as a CMAB problem, where the contexts represent multimodal physiological data streams obtained from sensor readings and the arms represent bolus insulin doses that are appropriate for injection, and show that blood glucose control can be significantly improved by using our method.

In a nutshell, our main contribution is to design an online learning algorithm that can maximize the cumulative expected reward (minimize the regret) in sequential decision-making problems that involve high-dimensional and large context and arm sets with a sparse structure, where the expected reward is a (possibly) non-linear function of contexts and arms. While doing so, we do not make any assumptions on how contexts arrive over time as stochastic models may fail to accurately capture real-world phenomena that generate the contexts. Nevertheless, we show that time-averaged regret can be made arbitrarily small by utilizing the prior knowledge which states that similar contexts and actions should yield similar expected rewards.

The rest of the paper is organized as follows. Related work is given in Section \ref{relatedWork}. CMAB and the regret are described in Section \ref{problemFormulation}. CMAB-RL is introduced in Section \ref{algorithm} and its regret is analyzed in Section \ref{regretAnalysis}. The effectiveness of learning the relevant dimensions is shown via simulations over (i) a high-dimensional synthetic dataset and (ii) a model created from real-world data collected from T1DM patients in Section \ref{simulations}. Concluding remarks are provided in Section \ref{conclusion} and appendices, including tables of notation and auxiliary results, are given in the supplemental document.

\section{Related Work} \label{relatedWork}

Research relevant to our work can be categorized along two dimensions: related work in CMAB and related work in relevance learning and dimension reduction.

\begin{table*}[t]
	\caption{Comparison of our work with the related works.}
	\label{table:related}
	\centering
	\begin{tabular}{l p{2.8cm} p{1.8cm}  p{2.5cm}  p{8cm}  }
		\toprule
		MAB algorithm & Regret bound &  Relevance learning & High-dimensional arm set  \\
		\midrule
		Contextual Zooming \cite{slivkins2014contextual}   & $ \tilde{O}(T^{1-1/(2+d_z)})$  & No & Yes      \\
		Query-Ad-Clustering \cite{lu2010contextual}             & $ \tilde{O}(T^{1-1/(2+d_c)})$ & No & Yes     \\
		RELEAF \cite{tekin2015releaf}      & $\tilde{O}(T^{g(\bar{d}_x)})$ & Yes & No    \\
		CMAB-RL (our work)  & $\tilde{O}(T^{1 - 1 /(2 + 2\overline{d}_x  + \overline{d}_a)})$ & Yes & Yes \\
		\bottomrule
	\end{tabular}
\end{table*}

\subsection{Related Work in CMAB}

CMAB has been studied under various assumptions on the relation between context-arm pairs and rewards. In the context of our work, prior art in CMAB can be categorized into three groups.

Problems in the first category (including our model) usually assume that there is an unknown but fixed reward distribution for every context-arm pair and the expected reward is a Lipschitz continuous function of the distance between context-arm pairs. Generally, for this category, no stochastic assumptions are made on the context arrivals. Under these assumptions, \cite{lu2010contextual} proposes an algorithm that achieves $O(T^{1-1/(2+d_c)+\epsilon})$ regret for any $\epsilon > 0$ where $d_c$ is the covering dimension of the similarity space, i.e., the space of feasible context-arm pairs. 
The proposed algorithm partitions the similarity space and uses the past history in each set of the partition to form reward estimates of context-arm pairs within that particular set. It is also shown that a lower bound of order $\Omega(T^{1-1/(2+d_p) -\epsilon})$ exists where $d_p$ is the packing dimension of the similarity space. Another related work \cite{slivkins2014contextual} proposes an algorithm that adaptively divides the similarity space with the help of a covering oracle, essentially by ``zooming" into regions where the context arrivals concentrate and arms provide high rewards, in order to perform high-precision exploration in these areas. It is shown that this algorithm achieves $\tilde{O}(T^{1-1/(2+d_z)})$ regret where $d_z$ is the zooming dimension, which is linked to the covering dimension of the set of near-optimal context-arm pairs. The same problem is considered in \cite{shekhar2018gaussian} with a Gaussian process prior on the reward, and a CMAB algorithm that constructs a tree of partitions inspired by the HOO strategy in \cite{bubeck2011x} is shown to achieve an optimal regret bound. To the best of our knowledge, the only other paper that considers relevance learning in this category is \cite{tekin2015releaf}. As noted in the introduction section, different from \cite{tekin2015releaf}, we consider a high-dimensional arm set and provide improved regret bounds by constructing a novel method to test the relevance.

The second category works under the linearly realizability assumption. Here, contexts represent arm features and the expected reward of an arm is a linear function of its context. \cite{li2010contextual} proposes LinUCB algorithm for personalized news article recommendation, and \cite{chu2011contextual} proves that a variant of LinUCB achieves $\tilde{O}(\sqrt{Td})$ regret, where $d$ is the dimension of the context. \cite{valko2013finite} extends these algorithms by introducing kernel functions, and shows that the proposed algorithm achieves $\tilde{O}(\sqrt{T\tilde{d}})$ regret, where $\tilde{d}$ represents the effective dimension of the kernel feature space. Notably, \cite{abbasi2011improved} provides an improved regret analysis for this problem by constructing more refined confidence sets. 
Sparsity in the context of linear CMAB is considered in \cite{wang2018minimax} and \cite{bastani2020online}. In these works, sparsity corresponds to having arm weight vectors with many zero elements, as dimensions with zero weights have no effect on the expected reward. Similar to our setting, these works also assume prior knowledge on sparsity in terms of an upper bound on the number of relevant dimensions. Unlike sparse linear CMAB, we consider sparsity in a much more general environment, where the reward is allowed to be a non-linear function of arms and contexts. We only impose a mild Lipschitz continuity assumption (Assumption \ref{simA}) on the expected reward, which allows our framework to be applicable to a much broader set of problems. We would also like to note that any linear bandit also satisfies the Lipschitz continuity assumption. Therefore, it can be said that \cite{wang2018minimax} and \cite{bastani2020online} assume a much stronger prior knowledge on the form of the expected reward than our work.

The third category assumes that at each round the context and the arm rewards in that round are jointly drawn from a time-invariant distribution and the goal is to compete with the best policy in a given policy class. Among many works that fall into this category, \cite{langford2007epoch} proposes the Epoch-Greedy algorithm that achieves $O(T^{2/3})$ regret. Follow-up works  such as \cite{agarwal2014taming} and \cite{dudik2011efficient} propose improved algorithms with  $\tilde{O}(T^{1/2})$ regret. 

Apart from these, \cite{feraud2016random} considers that each element of the context comes from a binary distribution and proposes the Bandit Forest algorithm. This algorithm chooses relevant contexts and eliminates the irrelevant ones by using conditional probabilities. However, it considers only finitely many arms and contexts. Learning the optimal policy from a logged dataset with bandit feedback is considered in \cite{atan2018constructing}. There, the authors identify the relevant context dimensions from logged data by constructing a relevance test that uses the importance sampling method. However, their method can only detect whether a context dimension is individually relevant or not. 

In addition to these, \cite{tyagi2016two} and \cite{djolonga2013high} investigate non-contextual MAB with high-dimensional arms. Like our work, \cite{tyagi2016two} assumes that only a subset of the arm dimensions are relevant and proposes a smart discretization of the arm set to achieve regret whose time order only depends on the number of relevant arm dimensions. On the other hand, \cite{djolonga2013high} assumes that the expected reward is low-dimensional and smooth, and proposes an explore-then-exploit strategy that performs subspace identification followed by Bayesian optimization to minimize the regret. Methods in these works cannot be directly applied in our setting since we also need to take into account exogenously arriving contexts. 

Table \ref{table:related} lists the assumptions and regret bounds of the works that are most closely related to ours.

\subsection{Related Work in Relevance Learning and Dimension Reduction}

Related work in relevance learning (or feature selection) mainly consists of offline methods. Similar to the related work in CMAB, offline feature selection can be categorized into three: Filter, wrapper and embedded approaches. In the embedded approach, feature selection is a part of the training procedure of a classifier. Wrapper methods select features based on the classifier's feedback. In contrast, filter methods do not take classifier feedback into account, and select features based on intrinsic and statistical properties of the features such as correlations and marginal distributions. A plethora of papers exist for each approach. 
For the embedded approach, decision trees \cite{deng2012feature} and lasso based methods \cite{tibshirani1996regression} are commonly used.
As an example of the wrapper methods, Recursive Feature Elimination proposed in \cite{guyon2002gene} iteratively trains the classifier, computes the ranking for each feature and removes the feature with smallest rank to find an optimal subset of the feature set.  Examples of filter methods include feature weighting \cite{kira1992practical} and information-theoretic feature selection algorithms \cite{peng2005feature}.

Online methods in feature selection can be seen as adaptations of offline methods. Due to computational efficiency, filter methods are generally preferred in the online framework  \cite{zhou2005streaming}. For instance, \cite{wu2010online} proposes a method called Online Streaming Feature Selection (OSFS). This algorithm divides the feature set into three disjoint sets: strongly relevant, weakly relevant and irrelevant. OSFS works in two phases. In the first phase, it learns strongly and weakly relevant features and eliminates irrelevant features. In the second phase, features that are relevant but redundant due to correlations with the other features are eliminated. While there is an abundance of literature in online feature selection (see e.g., \cite{yu2014towards} and references therein), they do not fit into the CMAB setting where the goal is to learn the relevant features in order to minimize the regret. Moreover, these works try to identify a fixed set of relevant features, while in our case the set of relevant context dimensions may differ among arms.

\section{Problem Formulation} \label{problemFormulation}
The system operates in rounds indexed by $t \in \{1,2,\ldots \}$. At the beginning of each round, the learner observes a context $x(t)$ that comes from a ${d}_x$-dimensional context set ${\cal X} \coloneqq [0,1]^{{d}_x}$, and then, chooses an arm $a(t)$ from a ${d}_a$-dimensional arm set ${\cal A} \coloneqq [0,1]^{{d}_a}$. The set of feasible context-arm pairs is denoted by ${\cal F} \coloneqq {\cal X} \times {\cal A}$. The random reward obtained from playing arm $a(t)$ in round $t$ is given as $r(t) \coloneqq \mu_{a(t)}(x(t)) + \kappa(t) $, where $\mu_{a}(x)$ denotes the expected reward of a context-arm pair  $ (x, a) \in {\cal F} $ and $\kappa(t)$ is the noise process whose marginal distribution is conditionally 1-sub-Gaussian, i.e. $\forall \lambda \in \mathbb{R}$
\begin{align*}
\expect{e^{\lambda \kappa (t)} | a_{1:t}, x_{1:t},\kappa_{1:t-1}} \leq \text{exp}(\lambda^2/2)
\end{align*}
where for $b \in\{a,x,\kappa\}$, $b_{1:t} \coloneqq (b(1),\ldots b(t))$. 

Let ${\cal D}_a \coloneqq \{1, \ldots, d_a\}$ denote the set of arm dimensions. 
For any $\bs{z} \subseteq {\cal D}_a$, ${\cal A}_{\bs{z}} \coloneqq [0,1]^{|\bs{z}|}$ denotes the subset of ${\cal A}$ that contains the values of arm dimensions in $\bs{z}$ and for any $a \in {\cal A}$, $a_{\bs{z}} \in {\cal A}_{\bs{z}}$ denotes the $|\bs{z}|$-tuple subarm whose elements are elements of $a$ that correspond to the arm dimensions in $\bs{z}$.
For any $\bs{z} \subseteq {\cal D}_a$ and $\bs{z}' = {\cal D}_a \setminus \bs{z}$, we write $ a = \{a_{\bs{z}}, a_{ \bs{z}'} \}$. 
Let $\bs{c}$ denote the subset of ${\cal D}_a$ that contains the relevant arm dimensions, i.e. $\forall \bs{z} \subseteq {\cal D}_{a} \setminus \bs{c}$, $\forall a_{\bs{z}}, a'_{\bs{z}} \in {\cal A}_{\bs{z}}$, $\forall a_{{\cal D}_a \setminus \bs{z}} \in {\cal A}_{{\cal D}_a \setminus \bs{z}} $ and $\forall x \in {\cal X}$, we have $\mu_{ \{a_{\bs{z}},  a_{{\cal D}_a \setminus \bs{z}}  \}}(x) = \mu_{ \{a'_{\bs{z}},  a_{{\cal D}_a \setminus \bs{z}} \} }(x)$.

Similarly, let ${\cal D}_x \coloneqq \{1, \ldots, d_x\}$ denote the set of context dimensions.
For any $\bs{z} \subseteq {\cal D}_x$, ${\cal X}_{\bs{z}} \coloneqq [0,1]^{|\bs{z}|}$ denotes the subset of ${\cal X}$ that contains values of the context dimensions in $\bs{z}$ and for any $x \in {\cal X}$, $x_{\bs{z}} \in {\cal X}_{\bs{z}}$ denotes the $|\bs{z}|$-tuple subcontext whose elements are elements of $x$ that correspond to the context dimensions in $\bs{z}$.
For any $\bs{z} \subseteq {\cal D}_x$ and $\bs{z}' = {\cal D}_x \setminus \bs{z}$, we write $x = \{x_{\bs{z}}, x_{ \bs{z}'} \}$.
Since relevant context dimensions may be different for different arms, for any $a \in {\cal A}$,  let $\bs{c}_a$ denote the subset of ${\cal D}_x$ that contains the relevant context dimensions, i.e. $\forall a \in {\cal A}$, $\forall \bs{z} \subseteq {\cal D}_{x} \setminus \bs{c}_a$, $\forall x_{\bs{z}}, x'_{\bs{z}} \in {\cal X}_{\bs{z}}$ and $\forall x_{{\cal D}_x \setminus \bs{z}} \in {\cal X}_{{\cal D}_x \setminus \bs{z}} $, we have $\mu_{a}(\{ x_{\bs{z}},  x_{{\cal D}_x \setminus \bs{z}} \}) = \mu_{a}(\{ x'_{\bs{z}},  x_{{\cal D}_x \setminus \bs{z}} \})$.

For a given context $x$, the optimal arm is defined as $a^*(x) \coloneqq \argmax_{a \in {\cal A}} \mu_a(x)$.
Since there are infinitely many arms and contexts, it is impossible to learn the optimal arm for each context without any further assumptions on the expected rewards. To overcome this issue, the following assumption provides a similarity structure on the expected rewards with respect to the set of context-arm pairs, which is a modified version of the Lipschitz continuity assumption commonly used in the contextual MAB literature \cite{slivkins2014contextual}.
It states that the variation of the expected reward between two context-arm pairs is bounded by the distance between the context-arm pairs in the relevant dimensions.

\begin{assumption} \label{simA}
$\exists L>0$ such that $\forall a, a' \in {\cal A}$ and $x,x' \in {\cal X}$, we have
\begin{align*}
|\mu_a(x) - \mu_{a'}(x')| \leq L (\lVert x_{\bs{c}_a} - x'_{\bs{c}_a}\rVert + \lVert a_{\bs{c}} - a'_{\bs{c}}\rVert)
\end{align*}
where $\lVert . \rVert$ represents the Euclidean norm. 
\end{assumption}

Assumption \ref{simA} also implies that
\begin{align*}
|\mu_a(x) - \mu_{a'}(x')| \leq L (\lVert x_{\bs{c}_{a'}} - x'_{\bs{c}_{a'}}\rVert + \lVert a_{\bs{c}} - a'_{\bs{c}}\rVert).
\end{align*}

We assume that the learner knows $L$ given in Assumption $\ref{simA}$, but does not know $\mu_a(x)$, $a \in {\cal A}$, $x \in {\cal X}$. To evaluate the performance of the learner given an arbitrary sequence of contexts $x_{1:T}$, we adopt the commonly used (pseudo) regret notion, given as
\begin{align*}
\text{Reg}(T) \coloneqq \sum_{t=1}^{T} \mu_{a^*(x(t))}(x(t))- \sum_{t=1}^{T} \mu_{a(t)}(x(t)) .
\end{align*}

Note that $\text{Reg}(T)$ is a random variable since $a(t)$ itself depends on the learning algorithm and its observations. In essence, $\text{Reg}(T)$ compares the expected reward accumulated by the learner with that of the oracle. Our goal is to design a learning algorithm to minimize the regret. Algorithms that do not take relevant dimensions into account (see, e.g. \cite{lu2010contextual}) will achieve $\tilde{O}(T^{1-1/(2+d_x+d_a)})$ regret in the worst-case. On the other hand, our algorithm CMAB-RL achieves $\tilde{O}(T^{1-1/(2+2\overline{d}_x + \overline{d}_a)})$ regret where $\overline{d}_x$ and $\overline{d}_a$ are known upper bounds on the number of relevant context and arm dimensions: $\underline{d}_x := \max_{a \in {\cal A}} |\bs{c}_a| \leq \overline{d}_x$ and $\underline{d}_a = |\bs{c}| \leq \overline{d}_a$. This shows that when $2\overline{d}_x + \overline{d}_a < d_x + d_a$, CMAB-RL achieves better regret compared to the algorithms that do not exploit the relevance structure. Thus, in the rest of the paper, we assume that $2\overline{d}_x \leq d_x$. Note that we do not require existence of a unique low-dimensional subspace of ${\cal F}$ that captures all the relevance, since it is possible that $\cup_{a \in {\cal A}} \bs{c}_a = {\cal D}_x$.

\section{The Learning Algorithm} \label{algorithm}
Our algorithm, called CMAB with Relevance Learning (CMAB-RL), is described in Algorithms \ref{algorithm:CMAB-RL} and \ref{algorithm:generate}. CMAB-RL is a CMAB algorithm that optimizes itself by generating supersets of the relevant context and arm dimensions with sizes $2\overline{d}_x$ and $\overline{d}_a$. The main step in learning relevance is to form a set of candidate dimensions (tuples) that contains the relevant dimensions with a high probability. Past observations that fall into these tuples are then used to estimate expected rewards of the arms, which results in highly accurate estimates when the tuples that contain the relevant dimensions are correctly identified. 

\begin{algorithm}[ht!]
	
	\caption{CMAB-RL}\label{algorithm:CMAB-RL}
	
	\begin{algorithmic}[1]
		
		\STATE Input: ${\cal X}, {\cal A}, T,  L, \overline{d}_x, \overline{d}_a$, $m$
		
		\STATE Initialization:(${\cal C}({\cal X}),{\cal Y}) = \text{Generate}({\cal X}, {\cal A}, \overline{d}_x, \overline{d}_a, m)$ \\
		Set $\hat{\mu}_{y,p_{\bs{w}}}(0) = 0 $, $N_{y,p_{\bs{w}}}(0) = 0$ for all $y \in {\cal Y}$, $\bs{w} \in {\cal V}^{2 \overline{d}_x}_{x}$, $p_{\bs{w}} \in {\cal P}_{\bs{w}}$
		
		\WHILE{$1 \leq t \leq T$}
		
		\STATE Observe $x(t)$ and for each $\bs{w} \in  {\cal V}^{2 \overline{d}_x}_{x}$, find $p_{\bs{w}}(t) \in {\cal P}_{\bs{w}}$ that $x(t)$ belongs to
		
		\STATE Compute ${\cal R}_y(t)$ for all $y \in {\cal Y}$ as given in \eqref{eqn:relset}
		
		\FOR {$y \in {\cal Y}$}
		
		\IF{${\cal R}_y(t) = \emptyset$}
		
		\STATE Randomly select $\hat{\bs{c}}_y (t)$ from $ {\cal V}^{\overline{d}_x}_{x}$
		
		\ELSE
		
		\STATE For each $\bs{v} \in {\cal R}_y(t)$, calculate $\hat{\sigma}^2_{y,\bs{v}}(t) = \max_{\bs{w}, \bs{w}' \in {\cal V}_x^{ 2\overline{d}_x } (\bs{v})}  | \hat{\mu}_{y,\bs{w}}(t) - \hat{\mu}_{y,\bs{w}'}(t) |$
		
		\STATE Set $\hat{\bs{c}}_y (t) = \argmin_{\bs{v} \in {\cal R}_y(t)} \hat{\sigma}^2_{y,\bs{v}}(t)$
		
		\ENDIF
		
		\STATE Calculate $\hat{\mu}^{\hat{\bs{c}}_y(t)}_{y}(t) = \frac{\sum\limits_{\bs{w} \in {\cal V}_x^{ 2\overline{d}_x } (\hat{\bs{c}}_y(t))} \hat{\mu}_{y,{\bs{w}}}(t) N_{y,{\bs{w}}}(t)}{\sum\limits_{\bs{w} \in {\cal V}_x^{ 2\overline{d}_x } (\hat{\bs{c}}_y(t))} N_{y,{\bs{w}}}(t)} $
		
		\STATE Determine $\bs{w}_y(t) = \argmax\limits_{\bs{w}' \in {\cal V}^{2 \overline{d}_x}_{x}} u_{y,\bs{w}'}(t)$
		
		\ENDFOR
		
		\STATE Select $y(t) = \argmax_{y \in {\cal Y}} \hat{\mu}^{\hat{\bs{c}}_y(t)}_{y}(t)  + 5 u_{y,{\bs{w}_{y}(t)}}(t)$
		
		\STATE Update estimates and the counters given for all $\bs{w} \in  {\cal V}^{2 \overline{d}_x}_{x}$
		
		\ENDWHILE
		
	\end{algorithmic}
	
\end{algorithm}

\begin{algorithm}[ht!]
	
	\caption{Generate}\label{algorithm:generate}
	
	\begin{algorithmic}[1]
		
		\STATE Input: ${\cal X}, {\cal A},\overline{d}_a, \overline{d}_x, m$
		
		\STATE Create ${\cal I}_i := \{[0, \frac{1}{m}], ( \frac{1}{m}, \frac{2}{m}], \ldots, (\frac{m-1}{m},1] \}$ and ${\cal P}_i := \{[0, \frac{1}{m}], ( \frac{1}{m}, \frac{2}{m}], \ldots, (\frac{m-1}{m},1] \}$
		
		\STATE Generate $ {\cal V}^{\overline{d}_a}_{a} $ and $ {\cal V}^{2\overline{d}_x}_{x} $
		
		\FOR {$\bs{v} \in {\cal V}^{\overline{d}_a}_{a}$}
		
		\STATE ${\cal I}_{\bs{v}} = \prod_{i \in \bs{v}} {\cal I}_i $
		
		\ENDFOR
		
		\FOR {$w \in {\cal V}^{2\overline{d}_x}_{x}$}
		
		\STATE ${\cal P}_{\bs{w}} = \prod_{i \in \bs{w}} {\cal P}_i $
		
		\ENDFOR
		
		\STATE ${\cal C}({\cal A}) = \bigcup_{ \bs{v} \in {\cal V}^{\overline{d}_a}_{a}} {\cal I}_{\bs{v}}$ and ${\cal C}({\cal X}) := \bigcup_{ \bs{w} \in {\cal V}^{2\overline{d}_x}_{x}} {\cal P}_{\bs{w}}$
		
		\STATE Index the geometric center of each set in ${\cal C}({\cal A})$ by $y$ and generate the set of arms $ {\cal Y}$
		
		\RETURN ${\cal C}({\cal X})$ and $ {\cal Y}$
		
	\end{algorithmic}
	
\end{algorithm}

For any $l \in \mathbb{Z}^{+}$, let ${\cal V}^l_x$ denote the set of all $l$-tuples of context dimensions, i.e. ${\cal V}^l_x \coloneqq \{ \bs{v} \in \wp( {\cal D}_x ): |\bs{v}| = l \}$ where $\wp ( {{\cal D}_x})$ denotes the power set (set of all subsets) of  ${{\cal D}_x}$.  Similarly for any $l \in \mathbb{Z}^{+}$, let ${\cal V}^l_a$ denote the set of all $l$-tuples of arm dimensions. For $\bs{v} \subseteq {\cal D}_x$ and $l \in \{ |\bs{v}|,  |\bs{v}| + 1, \ldots, \overline{d}_x \}$, let $ {\cal V}^{l}_{x} (\bs{v}) $ denote the set of all $l$-tuples of context dimensions that contain $\bs{v}$, i.e. if we have $\bs{w} \in {\cal V}^{l}_{x} (\bs{v})$, then $\bs{v} \subseteq \bs{w}$ is satisfied. 

At the beginning, CMAB-RL takes as inputs the context set ${\cal X}$, the arm set ${\cal A}$, the total number of rounds $T$, $L$ given in Assumption \ref{simA}, the partition number $m$ (which will be optimized later), an integer that is an upper bound on the number of relevant arm dimensions $\overline{d}_a \leq d_a$ and an integer that is an upper bound on the number of relevant context dimensions $\overline{d}_x \leq d_x/2$. CMAB-RL uses Assumption \ref{simA} to learn together for similar arms and similar contexts. This is achieved by properly discretizing the arm and context sets. In its initialization phase, CMAB-RL generates a discretized arm set ${\cal Y} \subseteq {\cal A}$ and a collection of partitions of ${\cal X}$, denoted by $C({\cal X})$ using the Generate subroutine given in Algorithm \ref{algorithm:generate}. 

Next, we describe this initialization process in detail. CMAB-RL first generates the set ${\cal V}^{\overline{d}_a}_{a}$. For all $\bs{v} \in {\cal V}^{\overline{d}_a}_{a}$, each dimension of the arm subset ${\cal A}_{\bs{v}}$ is partitioned into $m$ intervals with equal lengths. Letting ${\cal I}_i := \{[0, \frac{1}{m}], ( \frac{1}{m}, \frac{2}{m}], \ldots, (\frac{m-1}{m},1] \}$ denote the partition of the arm subset in dimension $i$, ${\cal I}_{\bs{v}} := \prod_{i \in \bs{v}} {\cal I}_i $ forms a partition of ${\cal A}_{\bs{v}}$ into $m^{\overline{d}_a}$ non-overlapping sets. The collection of partitions of the $\overline{d}_a$-dimensional subsets of the arm set formed this way is denoted by ${\cal C}({\cal A}) := \cup_{ \bs{v} \in {\cal V}^{\overline{d}_a}_{a}} {\cal I}_{\bs{v}}$. Note that ${\cal C}({\cal A})$ contains $\genfrac(){0pt}{1}{d_a}{\overline{d}_a}  m^{\overline{d}_a} $ sets.
We index the geometric centers of these sets by $y$, and the set of arms that correspond to these centers is denoted by ${\cal Y}$. For an arm that corresponds to the geometric center of a set in ${\cal I}_{\bs{v}} $, values of the dimensions of that arm in $i \in {\cal D}_a \setminus \bs{v}$ are set as $0.5$.\footnote{$0.5$ is chosen for convenience. Indeed, any value in $[0,1]$ will work.}

Similarly, CMAB-RL also generates the set ${\cal V}^{2 \overline{d}_x}_{x}$. For all $\bs{w} \in {\cal V}^{2 \overline{d}_x}_{x}$ each dimension of the context subset ${\cal X}_{\bs{w}}$ is partitioned into $m$ intervals with equal lengths. Letting ${\cal P}_i := \{[0, \frac{1}{m}], ( \frac{1}{m}, \frac{2}{m}], \ldots, (\frac{m-1}{m},1] \}$ denote the partition of the context subset in dimension $i$, ${\cal P}_{\bs{w}} := \prod_{i \in \bs{w}} {\cal P}_i $ forms a partition of ${\cal X}_{\bs{w}}$ into $m^{2 \overline{d}_x}$ non-overlapping sets. The collection of partitions of the $2 \overline{d}_x$-dimensional subsets of the context set formed this way is denoted by ${\cal C}({\cal X}) := \cup_{ \bs{w} \in {\cal V}^{2 \overline{d}_x}_{x}} {\cal P}_{\bs{w}}$. Note that ${\cal C}({\cal X})$ contains $\genfrac(){0pt}{1}{d_x}{2 \overline{d}_x}  m^{2 \overline{d}_x} $ sets. 

For simplicity of notation, for any $x \in {\cal X} $ if $x_{\bs{w}} \in p_{\bs{w}}$ for $p_{\bs{w}} \in {\cal P}_{\bs{w}}$, then we say that $x \in p_{\bs{w}}$ for $\bs{w} \in{\cal V}^{2 \overline{d}_x}_{x}$. Also, we let $p_{\bs{w}}(t) \in {\cal P}_{\bs{w}}$ denote the set that $x_{\bs{w}}(t)$ belongs to.

For each $\bs{w} \in {\cal V}^{2 \overline{d}_x}_{x}$, $p_{\bs{w}} \in {\cal P}_{\bs{w}}$ and $y \in {\cal Y}$, CMAB-RL stores a counter $N_{y, p_{\bs{w}}}(t)$ that counts the number of times context was in $p_{\bs{w}}$ and arm $y$ was selected before round $t$, and the sample mean of the rewards $\hat{\mu}_{y,p_{\bs{w}}}(t)$ that is obtained from rounds prior to round $t$ in which context was in $p_{\bs{w}}$ and arm $y$ was selected. 
In order to define the arm selection rule, CMAB-RL also needs to calculate another statistic, called the uncertainty term, which is defined for all $\bs{w} \in {\cal V}^{2 \overline{d}_x}_{x}$, $p_{\bs{w}} \in {\cal P}_{\bs{w}}$, $y \in {\cal Y}$ as $u_{y,p_{\bs{w}}}(t) \coloneqq \sqrt{(2+4\log (2 |{\cal Y}| \overline{C} m^{2 \overline{d}_x} T^{3/2}) ) / N_{y,p_{\bs{w}}}(t)} $, where $\overline{C} \coloneqq \genfrac(){0pt}{1}{d_x-1}{2 \overline{d}_x - 1}$. For simplicity of notation, we use $\hat{\mu}_{y,\bs{w}}(t) \coloneqq \hat{\mu}_{y, p_{\bs{w}}(t)}(t)$,  $u_{y,\bs{w}}(t) \coloneqq u_{y, p_{\bs{w}}(t)}(t)$ and $N_{y,\bs{w}}(t) \coloneqq N_{y,p_{\bs{w}}(t)}(t)$, since in each round $t$ there exists only one $p_{\bs{w}} \in {\cal P}_{\bs{w}}$ such that $x_{\bs{w}}(t) \in p_{\bs{w}}$. Based on this, the sample mean reward of arm $y \in {\cal Y}$ for the tuple of context dimensions $\bs{v} \in {\cal V}_x^{\overline{d}_x} $ in round $t$ is defined as
\begin{align*}
\hat{\mu}^{\bs{v}}_{y}(t) \coloneqq \frac{\sum\limits_{\bs{w} \in {\cal V}_x^{ 2\overline{d}_x } (\bs{v})} \hat{\mu}_{y,{\bs{w}}}(t) N_{y,{\bs{w}}}(t)}{\sum\limits_{\bs{w} \in {\cal V}_x^{ 2\overline{d}_x } (\bs{v})} N_{y,{\bs{w}}}(t)} . 
\end{align*}

At the beginning of round $t$, CMAB-RL first observes the context $x(t)$. Then, for each $\bs{w} \in {\cal V}_x^{2 \overline{d}_x}$, it identifies the set $p_{\bs{w}}(t)$ in ${\cal P}_{\bs{w}}$ that $x(t)$ belongs to. Using this information and the sample mean rewards, it generates the set of candidate relevant tuples of context dimensions for each $y \in {\cal Y}$ as follows:
\begin{align}
	&{\cal R}_y(t) \coloneqq \left\{ \bs{v} \in {\cal V}^{\overline{d}_x}_{x} : | \hat{\mu}_{y,\bs{w}}(t) - \hat{\mu}_{y,\bs{w}'}(t) | \right. \notag \\
	& \leq \left.
	2L \sqrt{\overline{d}_x} / m + u_{y, \bs{w}} (t) + u_{y, \bs{w}'} (t), 
	\forall \bs{w}, \bs{w}' \in   {\cal V}^{2\overline{d}_x}_{x} (\bs{v}) \right\} . \label{eqn:relset}
\end{align}

Here, the term $2L \sqrt{\overline{d}_x} / m +  u_{y, \bs{w}} (t) + u_{y, \bs{w}'} (t)$ accounts for the joint uncertainty over the sample mean rewards of arm $y$ calculated using observations in $p_{\bs{w}}(t)$ and $p_{\bs{w}'}(t)$. If the absolute difference between the sample mean rewards is larger than the joint uncertainty term, we can say that the subset of relevant context dimensions that is in tuple $\bs{w}$ is different from the subset of relevant context dimensions that is in tuple $\bs{w}'$ with high probability.
Since $\bs{v} \subset \bs{w}$ and $\bs{v} \subset \bs{w}'$, this implies that $\bs{v}$ does not contain all relevant context dimensions. Therefore, the tuple $\bs{v}$ is not included in the set of candidate relevant tuples of context dimensions ${\cal R}_y (t)$. 

Let $ \hat{\bs{c}}_y(t) $ denote the tuple of estimated relevant context dimensions for arm $y$ in round $t$. If ${\cal R}_y (t)$ is empty, then CMAB-RL selects $ \hat{\bs{c}}_y(t) $ from ${\cal V}_x^{\overline{d}_x}$ randomly. Otherwise, to compute $ \hat{\bs{c}}_y(t) $, CMAB-RL calculates the variation of the sample mean rewards for every $\bs{v} \in {\cal R}_y (t)$ as follows:
\begin{align*}
	\hat{\sigma}^2_{y, \bs{v}} (t) \coloneqq \max_{\bs{w}, \bs{w}' \in {\cal V}_x^{ 2\overline{d}_x } (\bs{v})}  | \hat{\mu}_{y,\bs{w}}(t) - \hat{\mu}_{y,\bs{w}'}(t) |. 
\end{align*}  
After calculating the variation, CMAB-RL chooses $ \hat{\bs{c}}_y(t) $ for all $y \in {\cal Y}$ as $\hat{\bs{c}}_y(t) = \argmin_{\bs{v} \in {\cal R}_y (t)} \hat{\sigma}^2_{y, \bs{v}} (t)$. 
Then, using $ \hat{\bs{c}}_y(t) $, CMAB-RL calculates $\mu_{y}^{ \hat{\bs{c}}_y(t) } (t)$ for all $y \in {\cal Y}$. To select an arm from ${\cal Y}$, CMAB-RL uses the principle of \textit{optimism under the face of uncertainty}. The estimated rewards of the context-arm pairs are inflated by a certain level, such that the inflated reward estimates become an upper confidence bound (UCB) for the expected reward with high probability. Denote the $2 \overline{d}_x $-tuple of context dimensions with the highest uncertainty term for arm $y$ in round $t$ by $\bs{w}_y(t)  \coloneqq \argmax_{\bs{w}' \in {\cal V}^{2 \overline{d}_x}_{x}} u_{y,\bs{w}'}(t)$ (where ties are broken randomly).
UCB of arm $y \in {\cal Y}$ at time $t$ is calculated as 
\begin{align*}
\text{UCB}_y(t) \coloneqq \hat{\mu}^{\hat{\bs{c}}_y(t)}_{y}(t)  + 5 u_{y,{\bs{w}_{y}(t)}}(t). 
\end{align*}
Then, CMAB-RL selects the arm with the highest UCB, i.e. $y(t) = \argmax_{y \in {\cal Y}} \text{UCB}_{y}(t) $. This forces the arms that are rarely selected by CMAB-RL to get explored (since they have high uncertainty) while balancing the trade-off between exploration and exploitation. After selecting arm $y(t)$, CMAB-RL observes the reward $r(t)$ and updates the parameters for arm $y(t)$ for all $\bs{w} \in {\cal V}_x^{ 2\overline{d}_x } $ as follows: 
\begin{align}
\hat{\mu}_{y(t),{\bs{w}}}(t+1) &= \frac{\hat{\mu}_{y(t),{\bs{w}}}(t) N_{y(t),{\bs{w}}}(t) + r(t) }{N_{y(t),{\bs{w}}}(t) + 1} \text{ and } \notag \\
N_{y(t),{\bs{w}}}(t+1) &= N_{y(t),{\bs{w}}}(t) + 1 . \label{ilgi1:update1}
\end{align}
In addition, for $y \in {\cal Y} \setminus y(t)$, $\bs{w} \in {\cal V}^{2\overline{d}_x}_x$ and $p_{\bs{w}} \in {\cal P}_{\bs{w}}$, we have $N_{y,p_{\bs{w}}}(t+1) = N_{y,p_{\bs{w}}}(t)$, $ \hat{\mu}_{y,p_{\bs{w}}}(t+1) = \hat{\mu}_{y,p_{\bs{w}}}(t)$, hence these values remain unchanged. Please refer to Appendix B in the supplemental document for the analysis of memory and computational complexities of CMAB-RL.

\begin{remark}
After a simple modification, CMAB-RL can also work when it is restricted to make choices from a given finite set of arms ${\cal A}_f$, which is a subset of the $d_a$-dimensional arm set ${\cal A}$. For this, it will first identify sets in ${\cal C}(A)$ that contain at least one arm in ${\cal A}_f$. Let ${\cal C}_f(A)$ represent the collection of such sets. For each set in ${\cal C}_f(A)$, CMAB-RL will pick a unique arm from ${\cal C}_f(A)$ and include it in ${\cal Y}$. By this construction, all arms in ${\cal Y}$ will be from ${\cal A}_f$. After initializing the arm set ${\cal Y}$ this way, CMAB-RL will compute and update UCB indices for these arms in the same way as the original algorithm. 
\end{remark}

\section{Regret Analysis} \label{regretAnalysis}

We first state and discuss our main result, and then, present the technical details. 

\subsection{The Main Result}

Our main result is given in the following theorem. 
	
\begin{theorem} \label{theorem:instreg1d}
	Given an arbitrary fixed sequence of contexts $x_{1:T}$, when CMAB-RL is run with $m = \lceil T^{1/(2+2\overline{d}_x + \overline{d}_a)} \rceil$, we have with probability at least $1 - 1/T$
	\begin{align*}
	& \text{Reg}(T) \leq  C_{\max} |{\cal V}_x^{2\overline{d}_x}| \genfrac(){0pt}{1}{d_a}{\overline{d}_a} \tilde{T}^{\frac{2\overline{d}_x + \overline{d}_a}{2+2\overline{d}_x + \overline{d}_a}} 
	\\ & \qquad + ( L (10\sqrt{\overline{d}_x}+\sqrt{\overline{d}_a}) 
	+ 2 \sqrt{   |{\cal V}_{2\overline{d}_x}| 
		\genfrac(){0pt}{1}{d_a}{\overline{d}_a}} B_{m,T} ) \tilde{T}^{\frac{1+2\overline{d}_x+\overline{d}_a}{2+2\overline{d}_x+\overline{d}_a}}
	\end{align*}
	where $\tilde{T} = (T^{1/(2+2\overline{d}_x+\overline{d}_a)} +1)^{2+2\overline{d}_x+\overline{d}_a}$ and $C_{\max} \coloneqq \max_{y, y' \in {\cal Y}, x \in {\cal X}} ( \mu_{y'}(x) - \mu_y(x))$.
\end{theorem}

Importantly, Theorem \ref{theorem:instreg1d} says that CMAB-RL incurs $\tilde{O}(T^{1 - 1 /(2 + 2\overline{d}_x  + \overline{d}_a)})$ regret with probability at least $1-1/T$ when it is run with $m =  \lceil T^{1/(2+2\overline{d}_x + \overline{d}_a)} \rceil$. A standard doubling trick argument \cite{bubeck2011x} can be used to make the algorithm anytime (does not require $T$ as input) while preserving the order of the regret. As a side result, this sublinear regret bound also implies average reward optimality of CMAB-RL. On the other hand, classical CMAB algorithms that do not exploit the relevance structure achieve $\tilde{O}(T^{1 - 1 /(2 + {d}_x  + {d}_a)})$ regret in the worst-case \cite{lu2010contextual}. Thus, when $2\overline{d}_x + \overline{d}_a < d_x + d_a$, CMAB-RL achieves a better regret order compared to the classical CMAB algorithms. As noted before, for the finite-armed version of our problem, RELEAF \cite{tekin2015releaf} achieves $\tilde{O}(T^{g(\bar{d}_x)})$ regret for $g(\bar{d}_x) = (2+2\bar{d}_x + \sqrt{4{\bar{d}_x}^2 + 16\bar{d}_x + 12}) / (4+2\bar{d}_x + \sqrt{4{\bar{d}_x}^2 + 16\bar{d}_x + 12})$, while our regret bound for this case becomes $\tilde{O}(T^{(1+2\bar{d}_x)/(2+2\bar{d}_x )})$, which is strictly better than that of RELEAF. As a final remark, we would also like to note that if $\bs{c}_a$ is fixed for all $a \in {\cal A}$, then it is possible to construct a strategy based on Exp4 \cite{auer2} that achieves $\tilde{O}(T^{1 - 1 /(2 + \overline{d}_x  + \overline{d}_a)})$ regret even though it requires defining an infeasible number of experts (see Appendix C in the supplemental document for details). In addition to assuming that the set of relevant context dimensions is the same for each arm, when the set of relevant context and arm dimensions are known (which is not the case in our work), an obvious lower bound on the worst-case regret would be $\Omega(T^{1 - 1 /(2 + \overline{d}_x  + \overline{d}_a)})$ \cite{lu2010contextual}. It is therefore an interesting future research direction to close the gap between this lower bound and our upper bound. 

\begin{remark}
It is also possible to consider a joint upper bound $\bar{d}_z$ on the number of relevant context and arm dimensions. In this case, since the learner does not know how many of these dimensions correspond to contexts or arms, it needs to consider all possible ways how $\bar{d}_z$-dimensions can be split between context and arms. Two extreme non-trivial cases are $(\bar{d}_x = \bar{d}_z -1, \bar{d}_a = 1)$ and $(\bar{d}_x = 0, \bar{d}_a = \bar{d}_z )$. Note that the case when $(\bar{d}_x = \bar{d}_z, \bar{d}_a = 0)$ is trivial as all arms in this case will yield the same expected reward for a given context, i.e., all arms are equally well and there is no need for learning. Thus, if only given $\bar{d}_z$, then the learner can set $\bar{d}_x = \bar{d}_z -1$ and $\bar{d}_a = \bar{d}_z$ in CMAB-RL. Based on Theorem 1, this will result in a regret bound of 
$\tilde{O}(T^{1 - 1/ (3 \bar{d}_z)})$ when $2 (\bar{d}_z - 1) \leq d_x$, which is still sublinear in $T$.	
\end{remark}

We end this subsection by giving a high-level explanation of the proof Theorem \ref{theorem:instreg1d}. To prove Theorem \ref{theorem:instreg1d}, as the first step, we construct contextual variants of the tight confidence sets derived from analysis of self-normalized martingale processes \cite{abbasi2011improved}. We build our analysis over concentration of these sets (intervals in our case) for the tuples that contain the relevant context dimensions. Our first result (Lemma \ref{lemma:prUC}) indicates that the confidence intervals remain reasonably small over all rounds with a high probability. The rest of our analysis focuses on what happens under this high probability event. For instance, defining the relevance test as given in \eqref{eqn:relset} ensures that all $\bar{d}_x$-tuples of context dimensions that include the relevant context dimensions pass the test (Lemma \ref{lemma:prgoodevent}), and this further guarantees that the estimated reward of each arm concentrates around its true mean value for the current context (Lemma \ref{lemma:3}). As a result of this, the UCB index used by CMAB-RL to select its arm ensures that the suboptimality gap of the selected arm is proportional to its uncertainty term (Lemma \ref{lemma:4}). As the uncertainty of an arm for the current context decreases every time that arm is selected, as time goes on, we conclude that the suboptimality gaps of the selected arms go to zero, which when summed over all rounds, gives us the worst-case regret bound. Technical details of the proof can be found in the next subsection. 

\subsection{Proof of Theorem \ref{theorem:instreg1d}}

We start by introducing the notation. For an event ${\cal H}$, let ${\cal H}^c$ denote its complement.
For any $\bs{w} \in {\cal V}_x^{2\overline{d}_x}$ and ${p}_{\bs{w}} \in {\cal P}_{\bs{w}}$, let $N_{{p}_{\bs{w}}}(t)$ denote the number of context arrivals to ${p}_{\bs{w}}$ by the end of round $t$, $\tau_{{p}_{\bs{w}}}(t)$ denote the round in which a context arrives to ${p}_{\bs{w}}$ for the $t$th time and $R_{y}(t)$ denote the random reward of arm $y$ in round $t$.

For any $\bs{w} \in {\cal V}_x^{2\overline{d}_x}$, ${p}_{\bs{w}} \in {\cal P}_{\bs{w}}$ and $y \in {\cal Y}$
let $\tilde{x}_{{p}_{\bs{w}}}(t) := x(\tau_{{p}_{\bs{w}}}(t))$, $\tilde{R}_{y,p_{\bs{w}}}(t) := R_{y}(\tau_{{p}_{\bs{w}}}(t))$, $\tilde{N}_{y,{p}_{\bs{w}}}(t) := N_{y,{p}_{\bs{w}}}(\tau_{{p}_{\bs{w}}}(t))$, 
$\tilde{\mu}_{y,{p}_{\bs{w}}}(t) := \hat{\mu}_{y,{p}_{\bs{w}}}(\tau_{{p}_{\bs{w}}}(t))$, 
$\tilde{u}_{y,{p}_{\bs{w}}}(t) := u_{y,{p}_{\bs{w}}}(\tau_{{p}_{\bs{w}}}(t))$.
$\tilde{y}_{{p}_{\bs{w}}}(t) := y(\tau_{{p}_{\bs{w}}}(t))$ and
$\tilde{\kappa}_{{p}_{\bs{w}}}(t) = \kappa(\tau_{{p}_{\bs{w}}}(t))$.

For any $\bs{v} \in {\cal V}_x^{\overline{d}_x }$ and $d' \leq d_x - \overline{d}_x$, $d' \in \mathbb{Z}^+$, let ${\cal V}_x (\bs{v},d')$ be the set of $d'$-tuples of context dimensions whose elements are from the set ${\cal D}_x \setminus \bs{v}$. Hence, for any $\bs{v} \in {\cal V}_x^{\overline{d}_x }$ and $\bs{j} \in {\cal V}_x (\bs{v},d')$, $(\bs{v}, \bs{j})$ denotes a ($\overline{d}_x + d'$)-tuple of context dimensions.

For any $y \in {\cal Y}$, $\bs{v} \in  {\cal V}_x^{\overline{d}_x }(\bs{c}_y)$, $\bs{j} \in {\cal V}_x (\bs{v}, \overline{d}_x)$ and $p_{(\bs{v}, \bs{j})} \in {\cal P}_{(\bs{v},\bs{j})}$ we define the following lower and upper bounds: $L_{y,{p}_{(\bs{v}, \bs{j})}}(t) := \tilde{\mu}_{y,{p}_{(\bs{v}, \bs{j})}}(t) - \tilde{u}_{y,{p}_{(\bs{v}, \bs{j})}}(t)$ and $U_{y,{p}_{(\bs{v}, \bs{j})}}(t) := \tilde{\mu}_{y,{p}_{(\bs{v}, \bs{j})}}(t) + \tilde{u}_{y,{p}_{(\bs{v}, \bs{j})}}(t)$.

For $\epsilon = L \left( \sqrt{\overline{d}_x}/m \right)$, $y \in {\cal Y}$, $\bs{v} \in  {\cal V}_x^{\overline{d}_x }(\bs{c}_y)$, $\bs{j} \in {\cal V}_x (\bs{v}, \overline{d}_x)$ and $p_{(\bs{v}, \bs{j})} \in {\cal P}_{(\bs{v},\bs{j})}$, let
\begin{align*}
\text{UC}_{y,{p}_{(\bs{v}, \bs{j})}} := 
\bigcup_{t=1}^{N_{{p}_{(\bs{v}, \bs{j})}}(T)} \{ \mu_y(\tilde{x}_{{p}_{(\bs{v}, \bs{j})}}(t)) &\notin \\ \notag [ L_{y,{p}_{(\bs{v}, \bs{j})}}(t) 
&- \epsilon , U_{y,{p}_{(\bs{v}, \bs{j})}}(t) + \epsilon ] \}
\end{align*}
denote the event that the learner is not confident about its reward estimate for at least once in time steps in which the contexts is in ${p}_{(\bs{v}, \bs{j})}$ by round $T$. 
Also, let $\text{UC}_{y,(\bs{v}, \bs{j})} \coloneqq \cup_{{p}_{(\bs{v}, \bs{j})} \in {{\cal P}}_{(\bs{v}, \bs{j})}} \text{UC}_{y,{p}_{(\bs{v}, \bs{j})}}, \text{UC}_{(\bs{v}, \bs{j})} \coloneqq \cup_{y \in {\cal Y}} \text{UC}_{y,(\bs{v}, \bs{j})}$ and 
\begin{align*}
\text{UC} \coloneqq \bigcup_{\bs{v} \in  {\cal V}^{\overline{d}_x}_x(\bs{c}_y), \bs{j} \in {\cal V}_x(\bs{v}, \overline{d}_x)} \text{UC}_{(\bs{v}, \bs{j})}.
\end{align*}
Similarly for any $y \in {\cal Y}$, $\bs{v} \in {\cal V}_x^{\overline{d}_x}(\bs{c}_y), \bs{j} \in {\cal V}_x (\bs{v}, \overline{d}_x)$ and $p_{(\bs{v}, \bs{j})} \in {\cal P}_{(\bs{v},\bs{j})}$, let 
\begin{align*}
\overline{\mu}_{y,{p}_{(\bs{v}, \bs{j})}} &= \sup_{x \in {p}_{(\bs{v}, \bs{j})}} \mu_y(x) ~\text{ and }~
\underline{\mu}_{y,{p}_{(\bs{v}, \bs{j})}} = \inf_{x \in {p}_{(\bs{v}, \bs{j})}} \mu_y(x) .
\end{align*}

The following lemma states that $\text{UC}$ occurs with a small probability. 
\begin{lemma} \label{lemma:prUC}
	\begin{align}
	\Pr ( \text{UC} ) \leq \frac{ 1 }{ T } .    \notag
	\end{align}
\end{lemma}
\begin{proof}
	Let $\{ \tilde{R}_{y,{p}_{(\bs{v}, \bs{j})}}(t) \}_{t=1}^{ N_{{p}_{(\bs{v}, \bs{j})}}(T) }$ denote the sequence of rewards observed from arm $y$ in time steps when the context is in ${p}_{(\bs{v}, \bs{j})}$.
	We can express the sample mean reward of $y$ as 
	\begin{align}
	\tilde{\mu}_{y,{p}_{(\bs{v}, \bs{j})}}(t) = \frac{ \sum_{l=1}^{t-1} \tilde{R}_{y,{p}_{(\bs{v}, \bs{j})}}(l)  \textbf{I}( \tilde{y}_{p_{(\bs{v}, \bs{j})}}(l) = y) }  {  \tilde{N}_{y, p_{(\bs{v}, \bs{j})}}(t)   }       \notag
	\end{align}
	for $\tilde{N}_{y, p_{(\bs{v}, \bs{j})}}(t) >0$, where $\textbf{I}(\cdot)$ is the indicator function. When $\tilde{N}_{y, p_{(\bs{v}, \bs{j})}}(t) = 0$ we have $\tilde{\mu}_{y,{p}_{(\bs{v}, \bs{j})}}(t) = 0$.
	We also have
	\begin{align}
	\tilde{R}_{y,{p}_{(\bs{v}, \bs{j})}}(t) = \mu_{y} ( \tilde{x}_{{p}_{(\bs{v}, \bs{j})}}(t))  + \tilde{\kappa}_{{p}_{(\bs{v}, \bs{j})}}(t)  \notag
	\end{align}
	where $\{ \tilde{\kappa}_{{p}_{(\bs{v}, \bs{j})}}(t) \}_{t=1}^{  N_{{p}_{(\bs{v}, \bs{j})}}(T)  }$ is a sequence of zero mean $1$-sub-Gaussian random variables. 
	We define two new sequences of random variables, whose sample mean values will lower and upper bound $\tilde{\mu}_{y,{p}_{(\bs{v}, \bs{j})}}(t)$. The {\em best sequence} is defined as 
	$\{  \bar{R}_{y,{p}_{(\bs{v}, \bs{j})}}(t) \}_{t=1}^{ N_{{p}_{(\bs{v}, \bs{j})}}(T) }$ where
	\begin{align}
	\overline{R}_{y,{p}_{(\bs{v}, \bs{j})}}(t) =   \overline{\mu}_{y,{p}_{(\bs{v}, \bs{j})}} +  \tilde{\kappa}_{{p}_{(\bs{v}, \bs{j})}}(t)        \notag
	\end{align}
	and the {\em worst sequence} is defined as 
	$\{  \underline{R}_{y,{p}_{(\bs{v}, \bs{j})}}(t) \}_{t=1}^{ N_{{p}_{(\bs{v}, \bs{j})}}(T) }$ where
	\begin{align}
	\underline{R}_{y,{p}_{(\bs{v}, \bs{j})}}(t) = \underline{\mu}_{y,{p}_{(\bs{v}, \bs{j})}} +  \tilde{\kappa}_{{p}_{(\bs{v}, \bs{j})}}(t)    .    \notag
	\end{align}
	Let 
	\begin{align}
	\overline{\mu}_{y,{p}_{(\bs{v}, \bs{j})}}(t) &:= \sum_{l=1}^{t-1} \overline{R}_{y,{p}_{(\bs{v}, \bs{j})}}(l)  \textbf{I}( \tilde{y}_{p_{(\bs{v}, \bs{j})}}(l) = y)  / \tilde{N}_{y, p_{(\bs{v}, \bs{j})}}(t)   \notag \\
	\underline{\mu}_{y,{p}_{(\bs{v}, \bs{j})}}(t) &:= \sum_{l=1}^{t-1} \underline{R}_{y,{p}_{(\bs{v}, \bs{j})}}(l)  \textbf{I}( \tilde{y}_{p_{(\bs{v}, \bs{j})}}(l) = y)  /  \tilde{N}_{y, p_{(\bs{v}, \bs{j})}}(t)   \notag 
	\end{align}
	for $\tilde{N}_{y, p_{(\bs{v}, \bs{j})}}(t) >0$. When $\tilde{N}_{y, p_{(\bs{v}, \bs{j})}}(t) = 0$ we have $\overline{\mu}_{y,{p}_{(\bs{v}, \bs{j})}}(t) = \underline{\mu}_{y,{p}_{(\bs{v}, \bs{j})}}(t) = 0$.
	Since $\bs{v} \in {\cal V}^{\overline{d}_x}_x(\bs{c}_y)$, we have $\forall t \in \{1,\ldots,N_{{p}_{(\bs{v}, \bs{j})}}(T)  \}  $
	\begin{align}
	\underline{\mu}_{y,{p}_{(\bs{v}, \bs{j})}}(t) \leq \tilde{\mu}_{y,{p}_{(\bs{v}, \bs{j})}}(t) \leq  \overline{\mu}_{y,{p}_{(\bs{v}, \bs{j})}}(t)  \notag
	\end{align}
	almost surely. 
	Let
	\begin{align}
	\overline{L}_{y,{p}_{(\bs{v}, \bs{j})}}(t) &:=  \overline{\mu}_{y,{p}_{(\bs{v}, \bs{j})}}(t) - \tilde{u}_{y,{p}_{(\bs{v}, \bs{j})}}(t)       \notag \\
	\overline{U}_{y,{p}_{(\bs{v}, \bs{j})}}(t) &:=  \overline{\mu}_{y,{p}_{(\bs{v}, \bs{j})}}(t) + \tilde{u}_{y,{p}_{(\bs{v}, \bs{j})}}(t)     \notag \\
	\underline{L}_{y,{p}_{(\bs{v}, \bs{j})}}(t) &:=   \underline{\mu}_{y,{p}_{(\bs{v}, \bs{j})}}(t) - \tilde{u}_{y,{p}_{(\bs{v}, \bs{j})}}(t)      \notag \\
	\underline{U}_{y,{p}_{(\bs{v}, \bs{j})}}(t) &:=  \underline{\mu}_{y,{p}_{(\bs{v}, \bs{j})}}(t) + \tilde{u}_{y,{p}_{(\bs{v}, \bs{j})}}(t)    .  \notag
	\end{align}
	Then, we have 
	\begin{align}
	& \{ \mu_{y}( \tilde{x}_{{p}_{(\bs{v}, \bs{j})}}(t) ) \notin [L_{y,{p}_{(\bs{v}, \bs{j})}}(t)- \epsilon, U_{y,{p}_{(\bs{v}, \bs{j})}}(t) + \epsilon]  \}  \notag \\
	& \subset \{ \mu_{y}( \tilde{x}_{{p}_{(\bs{v}, \bs{j})}}(t) ) 
	\notin [\overline{L}_{y,{p}_{(\bs{v}, \bs{j})}}(t)  - \epsilon, 
	\overline{U}_{y,{p}_{(\bs{v}, \bs{j})}}(t)  + \epsilon]  \}  \notag \\
	&\cup  \{ \mu_{y}( \tilde{x}_{{p}_{(\bs{v}, \bs{j})}}(t) ) 
	\notin [\underline{L}_{y,{p}_{(\bs{v}, \bs{j})}}(t)   - \epsilon, 
	\underline{U}_{y,{p}_{(\bs{v}, \bs{j})}}(t)  + \epsilon]  \} . \label{eqn:unionbound1}
	\end{align}

	The following inequalities are obtained using Assumption \ref{simA} since $\bs{v} \in {\cal V}^{\overline{d}_x}_x(\bs{c}_y)$:
	\begin{align}
	& \mu_{y}( \tilde{x}_{{p}_{(\bs{v}, \bs{j})}}(t) )  \leq \overline{\mu}_{y,{p}_{(\bs{v}, \bs{j})}}
	\leq \mu_{y}( \tilde{x}_{{p}_{(\bs{v}, \bs{j})}}(t) ) + \epsilon  \label{eqn:bestbound} \\
	& \mu_{y}( \tilde{x}_{{p}_{(\bs{v}, \bs{j})}}(t) )  - \epsilon   \leq \underline{\mu}_{y,{p}_{(\bs{v}, \bs{j})}}
	\leq \mu_{y}( \tilde{x}_{{p}_{(\bs{v}, \bs{j})}}(t) ) . \label{eqn:worstbound}  
	\end{align}

	Using \eqref{eqn:bestbound} and \eqref{eqn:worstbound} it can be shown that
	\begin{align}
	\{ \mu_{y}( \tilde{x}_{{p}_{(\bs{v}, \bs{j})}}(t) ) &\notin [\overline{L}_{y,{p}_{(\bs{v}, \bs{j})}}(t)  - \epsilon, \overline{U}_{y,{p}_{(\bs{v}, \bs{j})}}(t)    + \epsilon]  \}   \notag \\
	& \subset \{ \overline{\mu}_{y,{p}_{(\bs{v}, \bs{j})}} \notin [\overline{L}_{y,{p}_{(\bs{v}, \bs{j})}}(t)  , \overline{U}_{y,{p}_{(\bs{v}, \bs{j})}}(t) ]  \}, \notag \\
	\{ \mu_{y}( \tilde{x}_{{p}_{(\bs{v}, \bs{j})}}(t) ) &\notin [\underline{L}_{y,{p}_{(\bs{v}, \bs{j})}}(t) - \epsilon, \underline{U}_{y,{p}_{(\bs{v}, \bs{j})}}(t)  + \epsilon]  \}  \notag \\
	&\subset  \{ \underline{\mu}_{y,{p}_{(\bs{v}, \bs{j})}} \notin [\underline{L}_{y,{p}_{(\bs{v}, \bs{j})}}(t) , \underline{U}_{y,{p}_{(\bs{v}, \bs{j})}}(t) ]  \}  .  \notag
	\end{align}

	Plugging this to \eqref{eqn:unionbound1}, we get
	\begin{align*}
	\{ \mu_{y}( \tilde{x}_{{p}_{(\bs{v}, \bs{j})}}(t) ) &\notin [L_{y,{p}_{(\bs{v}, \bs{j})}}(t)- \epsilon, U_{y,{p}_{(\bs{v}, \bs{j})}}(t) + \epsilon]  \} \notag \\
	&\subset \{ \overline{\mu}_{y,{p}_{(\bs{v}, \bs{j})}} \notin [\overline{L}_{y,{p}_{(\bs{v}, \bs{j})}}(t)  , \overline{U}_{y,{p}_{(\bs{v}, \bs{j})}}(t) ]  \}
	\\ & \cup
	\{ \underline{\mu}_{y,{p}_{(\bs{v}, \bs{j})}} \notin [\underline{L}_{y,{p}_{(\bs{v}, \bs{j})}}(t) , \underline{U}_{y,{p}_{(\bs{v}, \bs{j})}}(t) ]  \}  
	\end{align*}
	Using the equation above and the union bound we obtain
	\begin{align}
	&\Pr( \text{UC}_{y,{p}_{(\bs{v}, \bs{j})}} ) \notag \\
	&\leq \Pr \left( \bigcup_{t=1}^{ N_{{p}_{(\bs{v}, \bs{j})}}(T) } \{ \overline{\mu}_{y,{p}_{(\bs{v}, \bs{j})}}  \notin [\overline{L}_{y,{p}_{(\bs{v}, \bs{j})}}(t)  , \overline{U}_{y,{p}_{(\bs{v}, \bs{j})}}(t) ]  \}  \right)   \notag \\
	&+ \Pr \left( \bigcup_{t=1}^{ N_{{p}_{(\bs{v}, \bs{j})}}(T) } \{ \underline{\mu}_{y,{p}_{(\bs{v}, \bs{j})}} 
	\notin [\underline{L}_{y,{p}_{(\bs{v}, \bs{j})}}(t)  , \underline{U}_{y,{p}_{(\bs{v}, \bs{j})}}(t)]  \} \right) . \notag
	\end{align}
	Both terms on the right-hand side of the inequality above can be bounded using the concentration inequality in Appendix D in the supplemental document by setting 
	$\delta = 1/ (2 |{\cal Y}| \overline{C}  m^{2\overline{d}_x} T)$:
	\begin{align*}
	\Pr( \text{UC}_{y,{p}_{(\bs{v}, \bs{j})}} ) \leq \frac{1}{   |{\cal Y}| \overline{C}  m^{2\overline{d}_x} T  } 
	\end{align*}
	since $1 + N_{y,{p}_{(\bs{v}, \bs{j})}}(T) \leq T$. Finally, the union bound gives us $\Pr( \text{UC} ) \leq 1/T$.  
\end{proof}

The next lemma states that ${\cal R}_y(t) \neq \emptyset$ for all $y \in {\cal Y}$ on event $\text{UC}^c$. 
\begin{lemma} \label{lemma:prgoodevent}
	On event $\text{UC}^c$, $\forall y \in {\cal Y}$, $\forall \bs{v} \in  {\cal V}_x^{\overline{d}_x }(\bs{c}_y)$ and $\forall t \in \{1,\ldots, T\}$, we have $\bs{v} \in {\cal R}_y(t)$. 
\end{lemma}
\begin{proof}
	$\forall y \in {\cal Y}$, $\forall \bs{v} \in  {\cal V}_x^{\overline{d}_x }(\bs{c}_y)$ and $\forall \bs{w} \in   {\cal V}^{2\overline{d}_x}_{x} (\bs{v})$, we have $\bs{w} \supset \bs{c}_{y}$, since $\bs{w} \supset \bs{v}$. By definition of $\text{UC}$, on event $\text{UC}^c$, $\forall t \in \{1,\ldots, T\}$, we have $|\hat{\mu}_{y,\bs{w}}(t) - \mu_y(x(t))| \leq \epsilon +  u_{y, \bs{w}} (t)$. Thus, $\forall \bs{w}, \bs{w}' \in  {\cal V}^{2\overline{d}_x}_{x} (\bs{v})$, we obtain $|\hat{\mu}_{y,\bs{w}}(t) - \hat{\mu}_{y,\bs{w}'}(t) | \leq 2\epsilon +  u_{y, \bs{w}} (t) + u_{y, \bs{w}'} (t)$ and consequently, we have $\bs{v} \in {\cal R}_y(t)$ by definition of ${\cal R}_y(t)$. 
\end{proof}

The next lemma shows that the difference between estimated and expected rewards of an arm is small on event $ \text{UC}^c $.

\begin{lemma} \label{lemma:3}
	On event $\text{UC}^c$, for all $y \in {\cal Y}$ and $t \in \{1,\ldots, T\}$ we have 
	\begin{align*}
	| \hat{\mu}^{\hat{\bs{c}}_y(t)}_{y}(t)  - \mu_y (x(t)) | \leq 5 \epsilon + 5 u_{y,{\bs{w}_y(t)}}(t) . 
	\end{align*}
	
\end{lemma}
\begin{proof}
	Fix $\bs{v} \in {\cal V}_x^{\overline{d}_x}(\bs{c}_y)$.
	Since $\bs{c}_y \subseteq \bs{v} $, we have on event $\text{UC}^c$
	\begin{align*}
		 \hat{\mu}^{\bs{v}}_{y}(t) &= \frac{\sum\limits_{\bs{w}' \in {\cal V}_x^{ 2\overline{d}_x } (\bs{v})} \hat{\mu}_{y,\bs{w}'}(t) N_{y,\bs{w}'}(t) }{\sum\limits_{\bs{w}' \in {\cal V}_x^{ 2\overline{d}_x } (\bs{v})} N_{y,\bs{w}'}(t) }\\ &\leq \notag 
		 \frac{\sum\limits_{\bs{w}' \in {\cal V}_x^{ 2\overline{d}_x } (\bs{v})} ( \mu_y (x(t))  + \epsilon +  u_{y, \bs{w}_y(t)}(t))  N_{y,\bs{w}'}(t) }{\sum\limits_{\bs{w}' \in {\cal V}_x^{ 2\overline{d}_x } (\bs{v})} N_{y,\bs{w}'}(t) } \notag \\
		 &= \mu_y (x(t))  + \epsilon +  u_{y,\bs{w}_y(t)}(t) . \notag
	\end{align*}
	Similarly, we also have
	\begin{align*}
		 \hat{\mu}^{\bs{v}}_{y}(t) &\geq 
		   \frac{\sum\limits_{\bs{w}' \in {\cal V}_x^{ 2\overline{d}_x } (\bs{v})} ( \mu_y (x(t))  - \epsilon -  u_{y, \bs{w}_y(t)}(t))  N_{y,\bs{w}'}(t) }{\sum\limits_{\bs{w}' \in {\cal V}_x^{ 2\overline{d}_x } (\bs{v})} N_{y,\bs{w}'}(t) } \notag \\
		 &= \mu_y (x(t))  - \epsilon -  u_{y,\bs{w}_y(t)}(t) . \notag
	\end{align*}
	Combining these two yields
	\begin{align}
	| \hat{\mu}^{\bs{v}}_{y}(t)   - \mu_y (x(t) ) | \leq \epsilon +  u_{y, \bs{w}_y(t)}(t). \label{eq1}
	\end{align}
	
Next, consider $\hat{\bs{c}}_y(t)$, which is chosen from ${\cal R}_y(t)$ as the $\overline{d}_x$-tuple of context dimensions with the minimum variation. We have for all $\bs{j}, \bs{k} \in {\cal V}_x ({\hat{\bs{c}}_y(t)}  , \overline{d}_x)$
\begin{align*}
	| \hat{\mu}_{y,(\hat{\bs{c}}_y(t), \bs{k})}(t) &- \hat{\mu}_{y,(\hat{\bs{c}}_y(t), \bs{j})}(t)| \leq
	\\ & 2 \epsilon + u_{y,(\hat{\bs{c}}_y(t), \bs{k})}(t) + u_{y,(\hat{\bs{c}}_y(t), \bs{j})}(t) .
\end{align*}
Also, on event $\text{UC}^c$, we have for all $\bs{l} \in {\cal V}_x(\bs{v},\overline{d}_x)$
\begin{align*}
    | \hat{\mu}_{y,(\bs{v}, \bs{l})}(t) -  \mu_y (x(t))  | \leq \epsilon + u_{y,(\bs{v}, \bs{l})}(t) .
\end{align*}
Thus, on event $\text{UC}^c$, we obtain for all $\bs{l}, \bs{n} \in {\cal V}_x(\bs{v},\overline{d}_x)$
\begin{align*}
    |\hat{\mu}_{y,(\bs{v}, \bs{l})}(t) - \hat{\mu}_{y,(\bs{v}, \bs{n})}(t)| &\leq 2 \epsilon +  u_{y,(\bs{v}, \bs{l})}(t) \notag +  u_{y,(\bs{v}, \bs{n})}(t) .
\end{align*}
Let $\bs{g}(\bs{v}, \hat{\bs{c}}_y(t))$ be a $2\overline{d}_x$-tuple of context dimensions that includes all entries of $\bs{v}$ and $\hat{\bs{c}}_y(t)$, i.e.,
for all $i \in \bs{v}$ and $j \in \hat{\bs{c}}_y(t) $, we have $i,j \in \bs{g}(\bs{v}, \hat{\bs{c}}_y(t))$. The existence of at least one such $2\overline{d}_x$-tuple of context dimensions is guaranteed since $\bs{v}$ and $\hat{\bs{c}}_y(t)$ are both $\overline{d}_x$-tuples of context dimensions. Combining what we have obtained thus far, we get
\begin{align*}
    |\hat{\mu}^{\bs{v}}_{y}(t) - \hat{\mu}^{\hat{\bs{c}}_y(t)}_{y}(t) |& \notag \\
    \leq \max\limits_{\substack{\bs{k} \in {\cal V}_x(\bs{v}, \overline{d}_x) \\ \bs{j} \in {\cal V}_x(\hat{\bs{c}}_y(t) , \overline{d}_x)}} &\Big\{ |\hat{\mu}_{y,(\bs{v}, \bs{k})}(t)
    - \hat{\mu}_{y,(\hat{\bs{c}}_y(t), \bs{j})}(t) | \Big\}  \\
    \leq \max\limits_{\substack{\bs{k} \in {\cal V}_x(\bs{v}, \overline{d}_x) \\ \bs{j} \in {\cal V}_x(\hat{\bs{c}}_y(t) , \overline{d}_x)}} &\Big\{ |\hat{\mu}_{y,(\bs{v}, \bs{k})}(t) - \hat{\mu}_{y, \bs{g}(\bs{v}, \hat{\bs{c}}_y(t))}(t) | \\ & + | \hat{\mu}_{y,\bs{g}(\bs{v}, \hat{\bs{c}}_y(t))}(t) - \hat{\mu}_{y,(\hat{\bs{c}}_y(t), \bs{j})}(t)| \Big\} \\
    \leq \max\limits_{\substack{\bs{k} \in {\cal V}_x(\bs{v}, \overline{d}_x) \\ \bs{j} \in {\cal V}_x(\hat{\bs{c}}_y(t) , \overline{d}_x)}} &\Big\{  4\epsilon + u_{y,(\bs{v}, \bs{k})}(t) \\ & + 
    u_{y,(\hat{\bs{c}}_y(t)), \bs{j})}(t) + 2 u_{y,\bs{g}(\bs{v}, \hat{\bs{c}}_y(t))}(t) \Big\}\\ 
    \leq 4 \epsilon + 4 u_{y,{\bs{w}_y(t)}}(t) .
\end{align*}

Finally, combining the result above with \eqref{eq1}, we obtain
	\begin{align*}
	| \hat{\mu}^{\hat{\bs{c}}_y(t)}_{y}(t)  - \mu_y (x(t)) | \leq 5 \epsilon + 5 u_{y,{\bs{w}_y(t)}}(t) . 
	\end{align*}
\end{proof}

To prove the next lemma, we introduce new notation. For $y \in {\cal Y}$, $\bs{w} \in {\cal V}_x^{2\overline{d}_x}$ and $p_{\bs{w}} \in {\cal P}_{\bs{w}}$, let 
\begin{align*}
{\cal T}_{y, \bs{w}, {p}_{\bs{w}}} & \coloneqq \{ t \in \{1,\ldots,T\} : x(t) \in {p}_{\bs{w}}, \ y(t) = y, \\
& \qquad \bs{w}_y(t) = \bs{w}  \} 
\end{align*}
and $\tau_{y, \bs{w}, {p}_{\bs{w}}}(t)$ denote the round in which a context arrives to ${p}_{\bs{w}}$, arm $y$ is chosen and $\bs{w}_y(t) = \bs{w} $ for the $t$th time. For simplicity, with an abuse of notation we let ${\cal T}_{y, {p}_{\bs{w}}} \coloneqq {\cal T}_{y, \bs{w}, {p}_{\bs{w}}}  $ and $\tau_{y, {p}_{\bs{w}}}(t) \coloneqq \tau_{y, \bs{w}, {p}_{\bs{w}}}(t)$.

\begin{lemma} \label{lemma:4}
	On event $\text{UC}^c$, for all $y \in {\cal Y}$, $\bs{w} \in {\cal V}_x^{2\overline{d}_x}$, $p_{\bs{w}} \in {\cal P}_{\bs{w}}$ and for all $t \in \{1,\ldots, |{\cal T}_{y,{p}_{\bs{w}}}| \}$, we have
	\begin{align*}
	\mu_{y^*(\typwt)}(x(\tau_{y,{p}_{\bs{w}}}(t))) &-  \mu_{y}(x(\tau_{y,{p}_{\bs{w}}}(t))) \\ & \leq 10 \epsilon + 10 u_{y,{\bs{w}}}(\typwt)
	\end{align*}
	where $y^*(t) \in \argmax_{y' \in {\cal Y}} \mu_{y'} ( x(t))$.
\end{lemma}

\begin{proof}
	Since CMAB-RL chooses arm $y$ in round $\tau_{y,{p}_{\bs{w}}}(t)$, we have $y \in \argmax_{y' \in {\cal Y}} \{ \hat{\mu}^{\hat{\bs{c}}_{y'}(\typwt)}_{y'}(\typwt) + 5 u_{y', \bs{w}_{y'}(\typwt)}(\typwt) \}$. By Lemma 3, we have
	\begin{align*}
	| \hat{\mu}^{\hat{\bs{c}}_y(\typwt)}_{y}(\typwt) &- \mu_y (x(\typwt)) | \\ & \leq 5 \epsilon + 5 u_{y,\bs{w}_{y}(\typwt)}(\typwt) .
	\end{align*}
	For all $y' \in {\cal Y}$, let
	\begin{align*}
	U'_{y'}(t) &\coloneqq  \hat{\mu}^{\hat{\bs{c}}_{y'}(t)}_{y'}(t) + 5 u_{y',\bs{w}_{y'}(t)}(t)  + 5 \epsilon \text{ and } \\
	L'_{y'}(t) &\coloneqq  \hat{\mu}^{\hat{\bs{c}}_{y'}(t)}_{y'}(t) - 5 u_{y',\bs{w}_{y'}(t)}(t)  - 5 \epsilon . 
	\end{align*}
	Note that by the selection rule of CMAB-RL, $U'_{y}(\typwt) \geq U'_{y^*(\tau_{y,{p}_{\bs{w}}}(t))}(\typwt)$.
	Combining this with the result of Lemma \ref{lemma:3} we obtain $U'_{y}(\typwt) \geq U'_{y^*(\tau_{y,{p}_{\bs{w}}}(t))}(\typwt) \geq \mu_{y^*(\typwt)}(x(\typwt)) \geq \mu_{y}(x(\typwt)) \geq L'_{y}(\typwt)$. 
	Therefore, we get $\mu_{y^*(\typwt)}(x(\tau_{y,{p}_{\bs{w}}}(t))) -  \mu_{y}(x(\tau_{y,{p}_{\bs{w}}}(t)))  \leq U'_y(\typwt) - L'_{y}(\typwt) = 10 \epsilon + 10 u_{y, \bs{w}_{y}(\typwt)}(\typwt)$.
	 Finally, note that in round $ \typwt $ it holds that $\bs{w}_y(\typwt) = \bs{w} $, hence we also have $ u_{y, \bs{w}_{y}(\typwt)}(\typwt) = u_{y, \bs{w}}(\typwt) $. Using this information we get the inequality stated in the lemma.
\end{proof}

For each $y \in {\cal Y}$, there are $|{\cal V}_x^{2\overline{d}_x}| = \genfrac(){0pt}{1}{d_x}{2\overline{d}_x}$ different $2\overline{d}_x$-tuples of context dimensions and for each $2\overline{d}_x$-tuple of context dimensions $\bs{w} \in {\cal V}_x^{2\overline{d}_x}$, $|{{\cal P}}_{\bs{w}}| = m^{2\overline{d}_x}$. Thus, we have
\begin{align*}
&\sum_{t=1}^{T} \mu_{y^*(x(t))}(x(t))- \sum_{t=1}^{T} \mu_{y(t)}(x(t)) \notag \\ &\leq C_{\max} |{\cal V}_x^{2\overline{d}_x}|m^{2\overline{d}_x} |{\cal Y}|  \\& 
+  \sum_{y \in {\cal Y}} \sum_{\bs{w} \in {\cal V}_x^{2\overline{d}_x}} \sum_{{p}_{\bs{w}} \in {{\cal P}}_{\bs{w}}} \sum_{t \in \{1,\ldots, |{\cal T}_{y,{p}_{\bs{w}}}| \}}  10u_{y,{\bs{w}}}(\typwt) + 10 \epsilon \\
&=  C_{\text{max}} |{\cal V}_x^{2\overline{d}_x}| m^{2\overline{d}_x} |{\cal Y}| + 10 \epsilon T  \\&  +  \sum_{y \in {\cal Y}} \sum_{\bs{w} \in {\cal V}_x^{2\overline{d}_x}} \sum_{{p}_{\bs{w}} \in {{\cal P}}_{\bs{w}}} \sum_{t \in \{1,\ldots, |{\cal T}_{y,{p}_{\bs{w}}}| \}} 10u_{y,{\bs{w}}}(\typwt)  \\
& \leq C_{\text{max}} |{\cal V}_x^{2\overline{d}_x}| m^{2\overline{d}_x} |{\cal Y}| + 10 \epsilon T  \\&  + B_{m,T} \sum_{y \in {\cal Y}} \sum_{\bs{w} \in {\cal V}_x^{2\overline{d}_x}} \sum_{{p}_{\bs{w}} \in {{\cal P}}_{\bs{w}}} \sum_{l=0}^{ |{\cal T}_{y,{p}_{\bs{w}}}| - 1 } \sqrt{ \frac{1}{1 +l} } \notag \\
& \leq C_{\text{max}} |{\cal V}_x^{2\overline{d}_x}| m^{2\overline{d}_x} |{\cal Y}| + 10 \epsilon T  \\&  + 2 B_{m,T} \sum_{y \in {\cal Y}} \sum_{\bs{w} \in {\cal V}_x^{2\overline{d}_x}} \sum_{{p}_{\bs{w}} \in {{\cal P}}_{\bs{w}}} \sqrt{|{\cal T}_{y,{p}_{\bs{w}}}|} \notag \\
& \leq C_{\text{max}} |{\cal V}_x^{2\overline{d}_x}| m^{2\overline{d}_x} |{\cal Y}| +  10 \epsilon T + 2 B_{m,T} \sqrt{  |{\cal V}_x^{2\overline{d}_x}| m^{2\overline{d}_x} |{\cal Y}| T   } 
\end{align*}
where $B_{m,T} \coloneqq 10\sqrt{2A_{m,T}}$ and $A_{m,T} \coloneqq (1 + 2 \log 
(2  |{\cal Y}| \bar{C}  m^{2\overline{d}_x} T^{3/2} ))$.

In order to bound the regret, next, we evaluate the error due to discretization of the arm set. Recall that instead of choosing arms from ${\cal A}$, CMAB-RL chooses arms from ${\cal Y}$ such that $|{\cal Y}| = m^{\overline{d}_a} \genfrac(){0pt}{1}{d_a}{\overline{d}_a}$. The regret due to this discretization can be bounded as 
\begin{align*}
\sum_{t=1}^{T} \mu_{a^*(x(t))}(x(t)) - \sum_{t=1}^{T} \mu_{y^*(x(t))}(x(t)) \leq T L \sqrt{\overline{d}_a}/m.
\end{align*}
Combining this with the regret bound obtained above and recalling that $\epsilon = L \left( \sqrt{\overline{d}_x}/m \right) $, we get
\begin{align*}
\text{Reg}(T) &\leq C_{\text{max}} |{\cal V}_x^{2\overline{d}_x}| m^{2\overline{d}_x} m^{\overline{d}_a} \genfrac(){0pt}{1}{d_a}{\overline{d}_a} + \frac{10L T}{m} \sqrt{\overline{d}_x}  \\& + \frac{L T \sqrt{\overline{d}_a}}{m} + 2 B_{m,T} \sqrt{  |{\cal V}_x^{2\overline{d}_x}| m^{2\overline{d}_x} m^{\overline{d}_a}\genfrac(){0pt}{1}{d_a}{\overline{d}_a} T   } 
\end{align*}
with probability $1- 1/T$. Finally, after choosing $m = \lceil T^{1/(2+2\overline{d}_x + \overline{d}_a)} \rceil$ the regret bound becomes
\begin{align*}
\text{Reg}(T) &\leq  C_{\text{max}} |{\cal V}_x^{2\overline{d}_x}|  \tilde{T}^{\frac{2\overline{d}_x + \overline{d}_a}{2+2\overline{d}_x + \overline{d}_a}} \genfrac(){0pt}{1}{d_a}{\overline{d}_a}  \\& + L (10\sqrt{\overline{d}_x}+\sqrt{\overline{d}_a}) \tilde{T}^{\frac{1+2\overline{d}_x+ \overline{d}_a }{2+2\overline{d}_x+\overline{d}_a}}& \notag \\
& + 2 B_{m,T} \sqrt{   |{\cal V}_x^{2\overline{d}_x}| \genfrac(){0pt}{1}{d_a}{\overline{d}_a}   } \tilde{T}^{\frac{1+2\overline{d}_x+\overline{d}_a}{2+2\overline{d}_x+\overline{d}_a}}&
\end{align*}
which proves Theorem \ref{theorem:instreg1d}.

\section{Illustrative Results} \label{simulations}
In this section, we numerically evaluate the performance of CMAB-RL in two experiments. In the first experiment, we generate a synthetic simulation environment with multi-dimensional context and arm sets, where in each set only a single dimension is relevant. In the second experiment, we apply CMAB-RL to the dynamic drug dosage regulation problem (bolus insulin administration) by utilizing OhioT1DM dataset \cite{Marling2018TheOD}.

\subsection{Competitor Learning Algorithms}\label{sec:competitors}    
    
 \subsubsection{Instance-based Uniform Partitioning (IUP) \cite{tekin2016confidence}} This is a contextual MAB algorithm that learns the optimal arm for each context by uniformly partitioning the set of feasible context-arm pairs $ {\cal F} $ into $ m^{d_x + d_a} $ hypercubes, where the choice $ m = \lceil T^{1/(2+d_x+d_a)} \rceil $ is shown minimize the regret. In each round, IUP first identifies the set of hypercubes that contain the current context, and then, plays an arm within the hypercube with the highest UCB among all hypercubes in that set. IUP does not take the relevance information into account. 
 
 \subsubsection{Contextual Hierarchical Optimistic Optimization (C-HOO)}
 This is the contextual version of hierarchical optimistic optimization (HOO) strategy proposed in \cite{bubeck2011x}.\footnote{Another related work \cite{shekhar2018gaussian} also proposes a contextual version of HOO for the Bayesian version of the MAB problem with Gaussian process prior.} Originally, HOO adaptively partitions the arm set $ {\cal A} $, by the help of a binary tree structure it stores. Each node of the tree corresponds to a subset of $ {\cal A} $, and as the depth level of a node increases, the subset it represents gets smaller. Subsets that correspond to nodes that have the same depth level form a partition on $ {\cal A} $. The tree of partitions is constructed in a way such that the union of the regions covered by the children of a node $ n $ is equal to the region that node $ n $ covers. In each round, HOO constructs a path starting from the root node, which corresponds to $ {\cal A} $. The path is constructed such that at every level of the tree, the child node with the highest UCB is added to the path. When a node with at most one child is reached, if the node has one child, the second child is created. Otherwise, a random child is created. The arm to be played is selected from the region that the newly created child represents. As HOO gathers information about the environment, it ``zooms" into regions with potentially high expected rewards, thereby performing more careful exploration in these regions.
 
We create C-HOO based on HOO as follows. First of all, we construct a tree of partitions over $ {\cal F} $ instead of $ {\cal A } $. In each round, C-HOO first observes the context, and then, constructs its path similar to HOO. The difference is that when constructing the path, at every level of the tree, first the availability (whether a node contains the context) of the children are checked, and among the children that contain the current context, the one with the highest UCB is added to the path.
It is also important to note that since the computational complexity of HOO increases quadratically with the number of rounds, we  construct C-HOO based on the truncated version of HOO \cite{bubeck2011x}, which is more efficient and enjoys the same regret bound as HOO except an additive factor of $ 4 \sqrt{T} $.

\subsubsection{Uniform Random} This benchmark randomly selects an arm in each round without taking the current context or past information into account. 

\subsection{Parameters Used in the Experiments}

We assume that the Lipschitz constants in both experiments are unknown to the learner, thus simply set $L = 1$ in the learning algorithms. Moreover, the set of all feasible context-arm pairs $ {\cal F} $, time horizon $ T $, dimensionality of context and arm sets, i.e., $ d_x $ and $ d_a $, are given as inputs to all learning algorithms. In addition, we set $ \overline{d}_x = \underline{d}_x$ and $\overline{d}_a = \underline{d}_a $ for CMAB-RL, and $ v_1 = 2\sqrt{d_x + d_a} $ and $ \rho=2^{(-1/(d_x + d_a))} $ for C-HOO (consistent with Assumption A1 in \cite{bubeck2011x}). For IUP, no additional parameters are required.
The confidence terms of all learning algorithms are scaled (multiplied) with a constant that is chosen from the set $ \{0.001, 0.005, 0.01, 0.05, 0.1, 0.25, 0.5, 1\} $ which pushes algorithms to exploit more. The rationale behind this choice is that during our experiments we observed that the confidence terms start large and vanish slowly forcing learning algorithms to explore too much, and scaling helps learning algorithms achieve higher cumulative rewards. For each learning algorithm, the optimal multiplier for the confidence term is found by grid search.  For all experiments, in order to reduce the effect of randomness due to context arrivals, arm selections and reward generation on the performance measurements, the reported results correspond to the average of 20 independent repetitions.

\subsection{Experiments on  a Synthetic Simulation Environment}
 	We consider a setting with $ d_x = 5 $,  $ d_a = 5 $, $ \underline{d}_x = 1 $ and $ \underline{d}_a = 1 $, and assume that the relevant context dimension is the same for all arms. We let the relevant arm and context dimensions to be the first arm and context dimensions respectively, i.e., $\bs{c}= \{1 \}$ and $\bs{c}_a = \{1\}$, $\forall a \in {\cal A}$. Since the expected reward function does not depend on the irrelevant context dimensions, we have $ \underline{d}_x + \underline{d}_a = 2 $. The expected reward function is defined by using a multivariate Gaussian mixture model, where the expected reward for context-arm pair $ (x, a) \in {\cal F} $ is given as
 	\begin{align*}
 		\mu_a(x) &= \min \left\{s\sum_{i = 1}^{K} \rho_i f( (x_1,a_1) | \theta_i, \Sigma_i), 1 \right\} 
 	\end{align*}
 	for $\sum_{i = 1}^{K} \rho_i = 1$ and $\rho_i > 0$, for $1 \leq i \leq K$. Here, $ s $ denotes the scaling factor, $ K $ denotes the number of components, $ f $ denotes the probability density function of a multivariate Gaussian distribution and $ \rho_i, \theta_i \text{ and } \Sigma_i $  stand for the component weight, mean vector and covariance matrix of the $i$th component, respectively. The parameters of the Gaussian mixture are set as follows:
 	$ s = 0.25 $,  $ K = 2 $, $ \rho_1 = \rho_2 = 0.5 $,  $\theta_1 = [0.25, 0.75]^T$, $\theta_2 = [0.5, 0.5]^T$ and
 	\begin{align*} 
     	\Sigma_1 = \begin{bmatrix}
    		 		0.05 & 0.03 \\
    		 		0.03 & 0.025 
     				\end{bmatrix}
    	&\text{,      }
    	\Sigma_2 = \begin{bmatrix}
    				 0.025 & -0.03 \\
    				 -0.03 & 0.05 
    				 \end{bmatrix}.
 	\end{align*}
 	Variation of the expected reward function over the relevant context and arm dimensions can be seen in Fig. \ref{fig:mixture}. The reward that the learner receives in round $ t $ is sampled from a Bernoulli distribution with parameter $\mu_{a(t)}(x(t))$ independently from the other rounds.

\begin{figure}[!t]
\centering
          \includegraphics[width=1.0\linewidth]{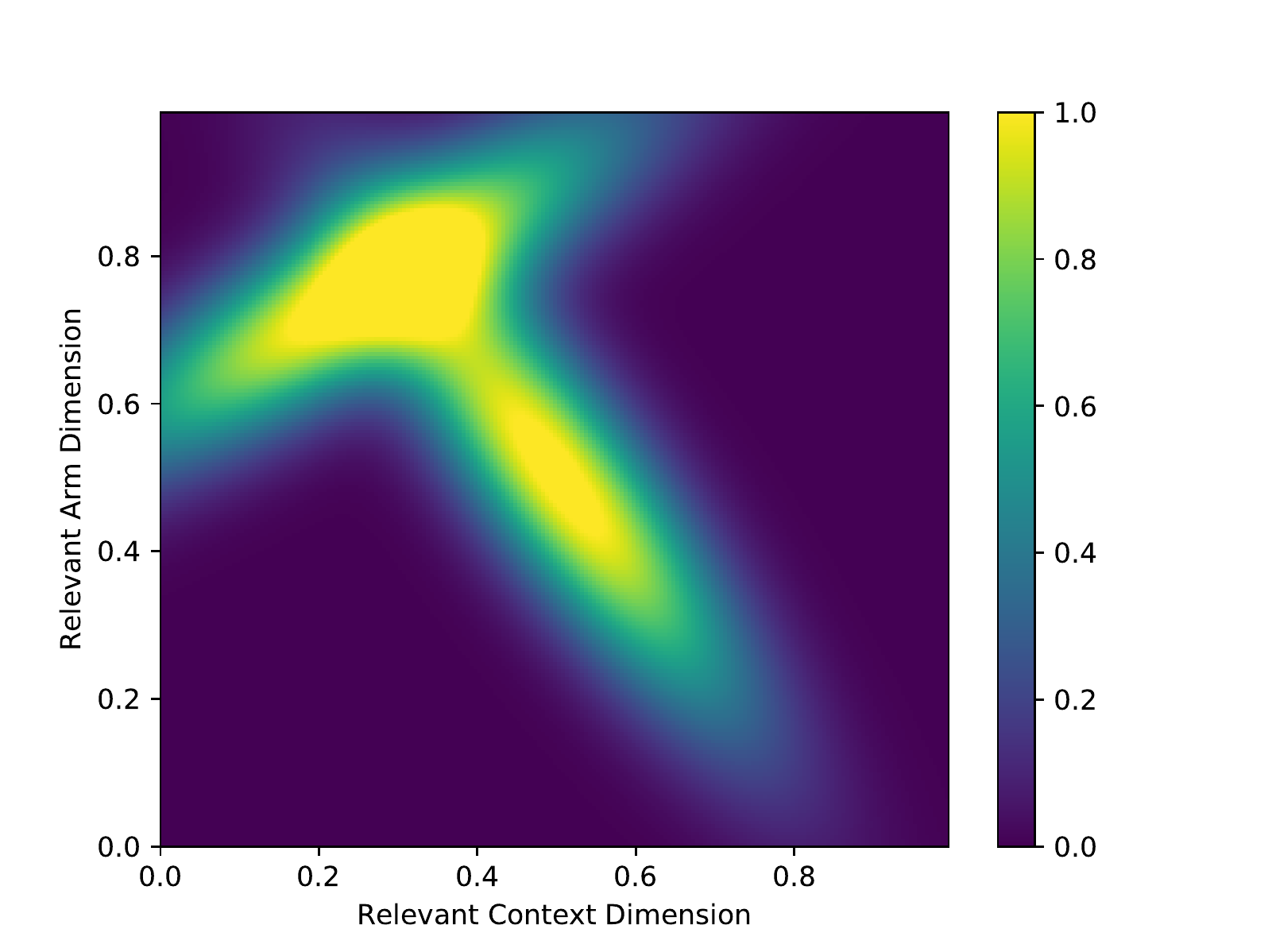}
          \caption{The expected reward as a function of the relevant context and arm dimensions in the first experiment.}
          \label{fig:mixture}
\end{figure}

\begin{figure}[!t]
\centering
        \includegraphics[width=1.0\linewidth]{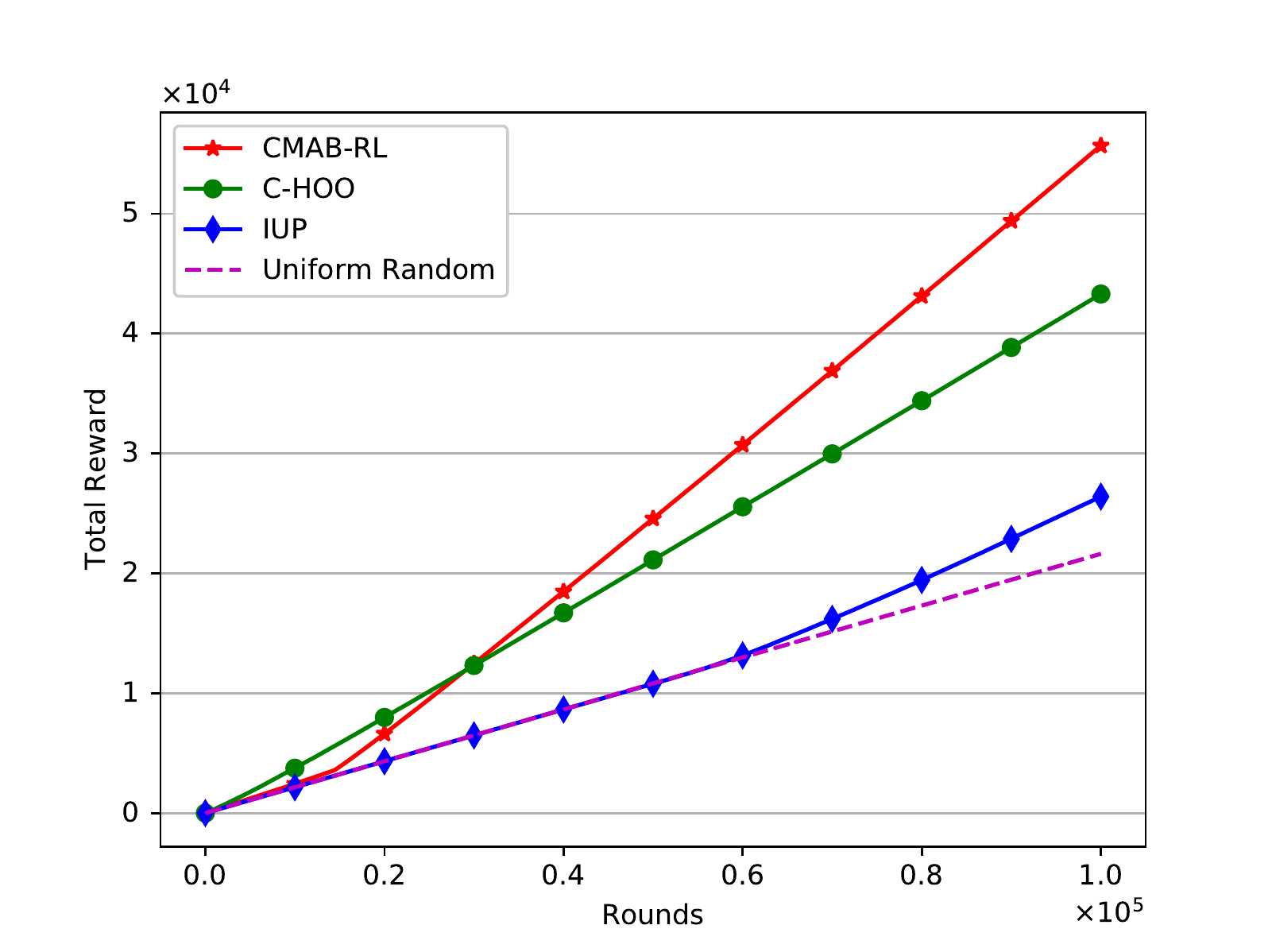}
        \caption{Cumulative rewards of CMAB-RL, C-HOO and IUP for $T=10^5$ in the first experiment.}
        \label{fig:gmm_result}
\end{figure}
 	
   	Learning algorithms are run for a time horizon of $ T = 10^5 $ rounds. In each round, a context arrives uniformly at random. The optimal multipliers for the confidence terms are found to be $0.001$ for CMAB-RL, $0.01$ for IUP and $0.05$ for C-HOO. Reported results correspond to this choice of multipliers. Cumulative rewards of the algorithms over time are given in Fig. \ref{fig:gmm_result}. As we can see, CMAB-RL achieves more than $29\%$ and $100\% $ improvement over the cumulative rewards of C-HOO and IUP respectively. Although C-HOO does not utilize relevancy information, it significantly outperforms IUP as a result of employing adaptive exploration using a tree of partitions. On the other hand, IUP performs poorly due to the curse of dimensionality. As a result, its cumulative reward is only slightly higher than that of Uniform Random. 
   	
   	Results on the regret are given in Fig. \ref{fig:gmm_regret}. The increase in the regret of CMAB-RL significantly drops down after $15000$ rounds, while the increase in the regrets of C-HOO and IUP does not drop significantly in the given time horizon. Since $T$ is an input to the learning algorithms, we provide additional results on the regret when the algorithms are run with input time horizons ranging from $T=5000$ to $T=10^5$. Fig. \ref{fig:sublinearity_regret} shows that CMAB-RL achieves the smallest regret for all time horizons. 
   
\begin{figure}[!t]
\centering
      \includegraphics[width=1.0\linewidth]{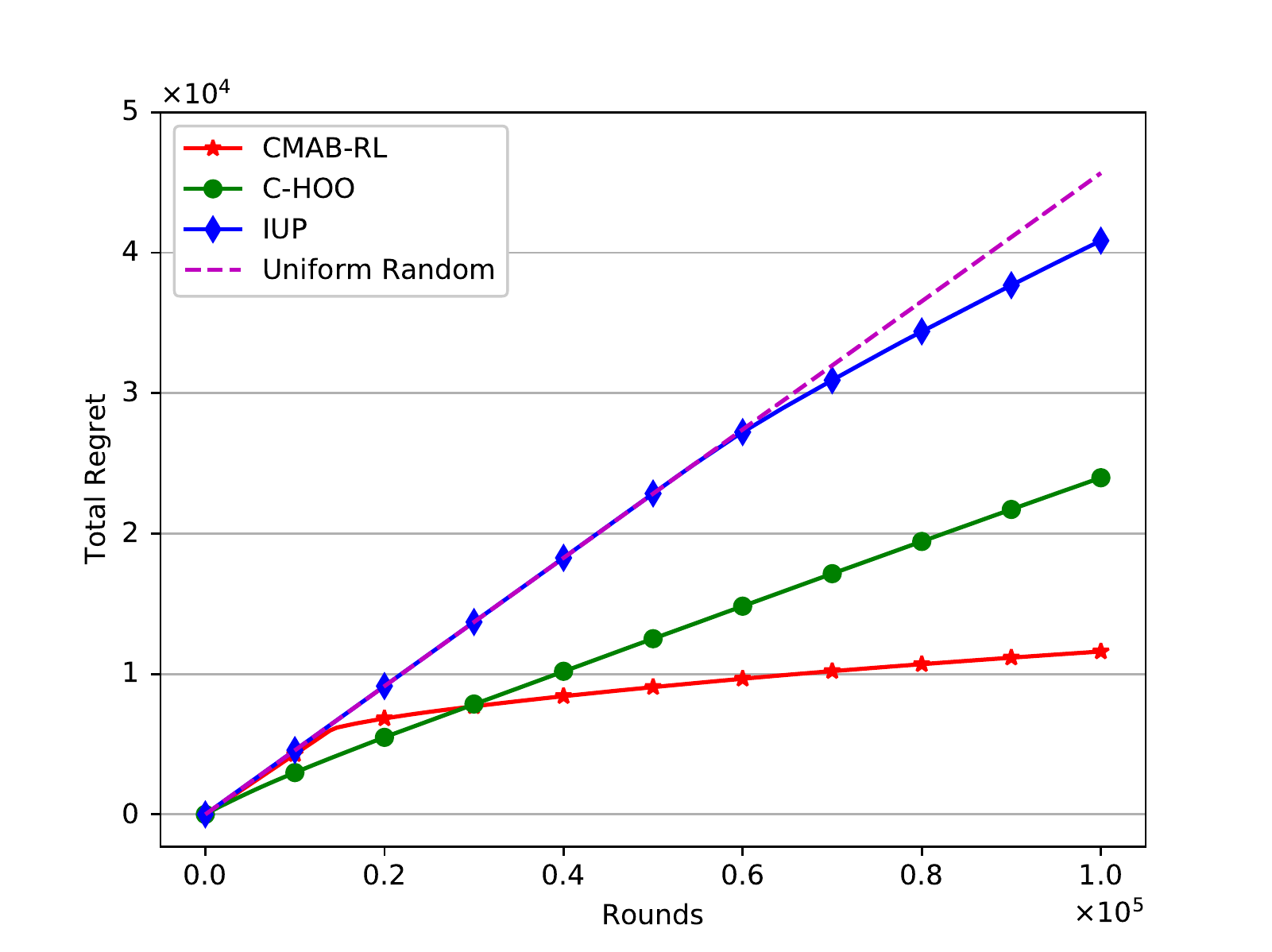}
      \caption{Regrets of CMAB-RL, C-HOO and IUP for $T=10^5$ in the first experiment.}
      \label{fig:gmm_regret}
\end{figure}

\begin{figure}[!t]
\centering
      \includegraphics[width=1.0\linewidth]{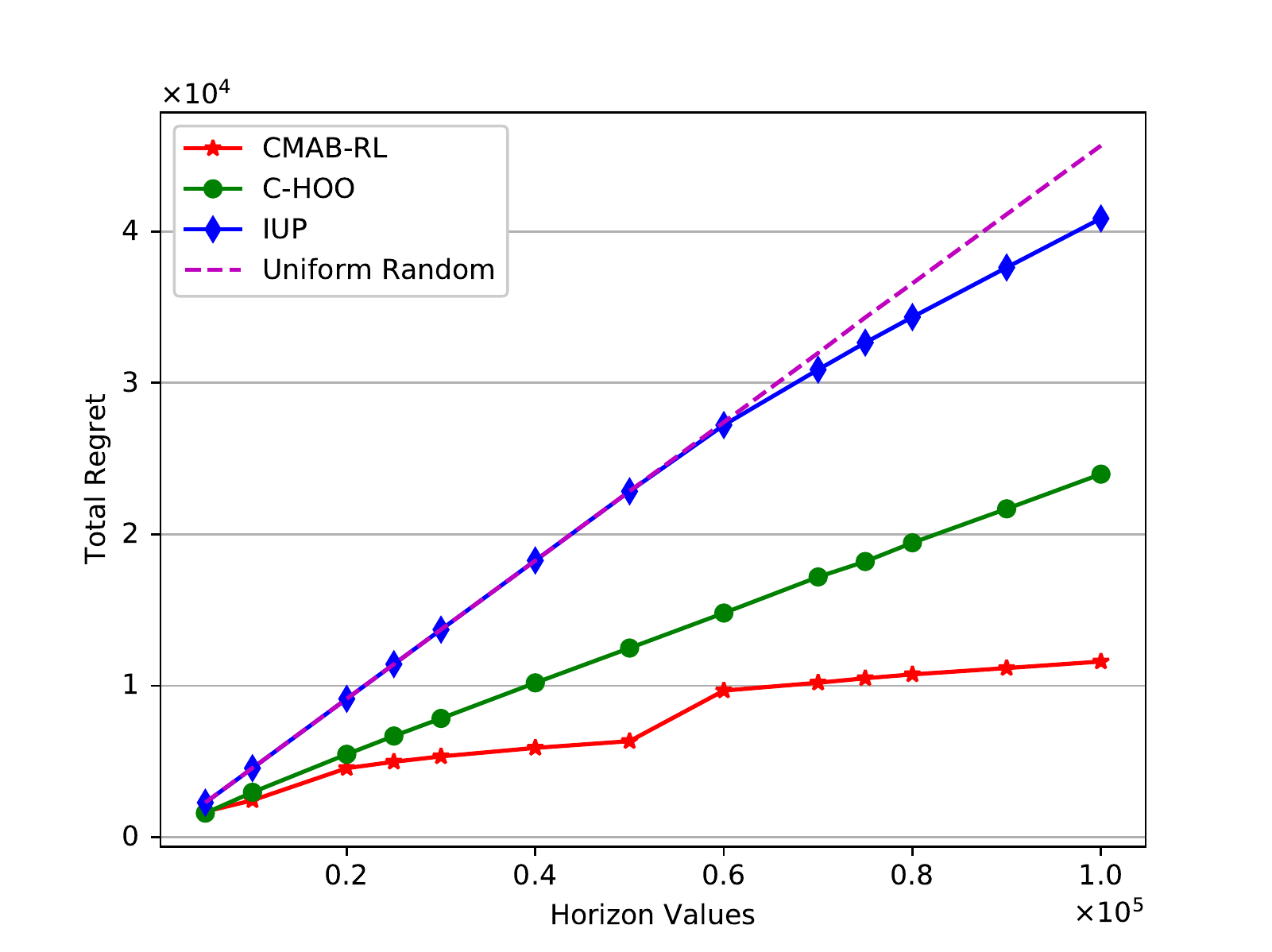}
      \caption{Regrets of CMAB-RL, C-HOO and IUP when they are run with different time horizons in the first experiment. The jumps in the regrets correspond to time horizons for which the value of $ m $ changes (since $m$ takes integer values).}
      \label{fig:sublinearity_regret}
\end{figure}

 	\begin{table*}[h!]
 		\centering
 		\caption{Percentages of samples for all approaches and patients}
 		\label{hist_table}
 		\small
 		\renewcommand{\arraystretch}{1.2}
 		\begin{tabular}{cc|c|c|c|c|c|c|c|}
 			
 			\cline{3-9}
 			&  & Patient 559 & Patient 563 & Patient 570 & Patient 575 & Patient 588 & Patient 591 & Overall \\ \hline
 			\multicolumn{1}{|c|}{\multirow{4}{*}{\begin{tabular}[c]{@{}c@{}}\textless{}80\\ mg/dL\end{tabular}}}
 			& CMAB-RL    & 00.25 & 00.04 & 00.13 & 00.47 & 00.24 & 00.71 & 00.27 \\ \cline{2-9} 
 			\multicolumn{1}{|c|}{} & C-HOO   & 00.30 & 00.05 & 00.14 & 00.58 & 00.32 & 00.89 & 00.34 \\ \cline{2-9} 
 			\multicolumn{1}{|c|}{} & IUP     & 00.35 & 00.05 & 00.18 & 00.69 & 00.38 & 00.96 & 00.38 \\ \cline{2-9} 
 			\multicolumn{1}{|c|}{} & Dataset & 01.97 & 00.25 & 01.59 & 04.50 & 00.00 & 03.41 & 01.75 \\ \hline
 			\multicolumn{1}{|c|}{\multirow{4}{*}{\begin{tabular}[c]{@{}c@{}}80-180\\ mg/dL\end{tabular}}} 
 			& CMAB-RL    & 66.92 & 70.49 & 66.64 & 88.65 & 69.06 & 79.31 & 72.78 \\ \cline{2-9} 
 			\multicolumn{1}{|c|}{} & C-HOO   & 53.57 & 62.54 & 54.98 & 81.54 & 56.73 & 67.27 & 62.30 \\ \cline{2-9} 
 			\multicolumn{1}{|c|}{} & IUP     & 50.10 & 56.08 & 51.76 & 77.11 & 52.39 & 64.25 & 57.99 \\ \cline{2-9} 
 			\multicolumn{1}{|c|}{} & Dataset & 36.84 & 56.78 & 37.93 & 62.00 & 39.81 & 57.95 & 49.06 \\ \hline
 			\multicolumn{1}{|c|}{\multirow{4}{*}{\begin{tabular}[c]{@{}c@{}}\textgreater{}180\\ mg/dL\end{tabular}}} 
 			& CMAB-RL    & 32.83 & 29.47 & 33.23 & 10.88 & 30.70 & 19.98 & 26.95 \\ \cline{2-9} 
 			\multicolumn{1}{|c|}{} & C-HOO   & 46.12 & 37.41 & 44.88 & 17.87 & 42.95 & 31.84 & 37.36 \\ \cline{2-9} 
 			\multicolumn{1}{|c|}{} & IUP     & 49.56 & 43.87 & 48.06 & 22.20 & 47.23 & 34.79 & 41.63 \\ \cline{2-9} 
 			\multicolumn{1}{|c|}{} & Dataset & 61.18 & 42.96 & 60.48 & 33.50 & 60.19 & 38.64 & 49.19 \\ \hline
 		\end{tabular}
 	\end{table*}

 \subsection{Experiments on the OhioT1DM dataset}
     For our second experiment, we use the OhioT1DM dataset that consists of several physiological measurements for 6 T1DM patients who are on continuous glucose monitoring and insulin pump therapy over a time period of 8 weeks (see \cite{Marling2018TheOD} for the details). While the original dataset is split into training and test sets for each patient in advance, we merge them into a single set to perform online learning.
 
     Our aim in this experiment is to learn the optimal bolus insulin dose for a patient such that their mean blood glucose levels remain within the desired range of $80$ to $180$ mg/dL (see, e.g., \cite{ada_guide}) by making use of contextual information such as the state of the patient and the ongoing basal insulin treatment before a bolus injection.
     As the state of the patient, we consider means of (i) continuous glucose measurements (CGMs), (ii) heart rate, (iii) skin temperature, (iv) air temperature and (v) galvanic skin response measurement, and sums of (i) carbohydrate intake from meals, (ii) exercise scores (multiplication of the duration and the intensity of an exercise session) and (iii) number of steps taken for the last $30$ minutes before a bolus injection. As the ongoing basal insulin treatment, we consider the mean of the basal insulin dosages for the last $30$ minutes. This corresponds to the setting where $ d_x = 9 $. As the arms, we only consider the bolus insulin dosages, thus $ d_a = \underline{d}_a = 1 $. Since bolus insulin doses are administered by an insulin pump that provides doses of insulin with a fine granularity, the set of bolus insulin doses can be approximated well by a continuum of values. Note that data is scaled such that it resides in range $ [0, 1] $ for all context and arm dimensions. 
     
     The rewards are based on the mean of the CGMs of the patients for the next 30 minutes to 2 hours after a bolus injection. Thus, for the sake of simplicity, in the rest of this section, we call CGM values that we use as contexts as {\em past} CGMs and CGM values that we use for reward generation as {\em resulting} CGMs. 
     
	We impute the missing values as follows. If no data is available to generate the contexts, then we set the contexts for carbohydrate intake, exercise and number of steps as zero, since lack of data suggests no activity. For heart rate, skin temperature, air temperature and galvanic skin response, we take the mean value of the whole dataset. Data is always available for bolus injections as we first locate the bolus events and extract other variables near the bolus events. If however, no data is available for past or resulting CGMs of a bolus event, then we ignore that bolus event.

    In order to setup the simulation, for each patient we fit a multivariate Gaussian distribution to all context dimensions, using only the said patient's data. Moreover, we learn a prior distribution over the patients by considering how frequently they appear in the dataset. We also need to model every possible combination of contexts, arms and rewards, which means that we need to learn a mapping from the context-arm space to the reward space. To achieve this, we use a Gradient Boosting regression model with Huber loss, which has $100$ decision trees as weak estimators where each tree is constrained to have a maximum depth of $5$. The inputs to the regression model are contexts and arms, whereas the outputs are the resulting CGMs. We use oversampling so that all patients have equal amount of data prior to the training of Gradient Boosting. The oversampling is done by sampling with replacement. During the experiment, in each round $t$, we select a patient randomly using the prior distribution, then we sample the context vector $x(t)$ from the selected patient's Gaussian distribution. If the generated context is not in range $[0,1]^{(dx + da)}$, we repeat the sampling process until a valid context is generated. Then, we feed the generated context to the CMAB algorithm. When the CMAB algorithm returns the arm $a(t)$, we query the regression model for the reward $ r(t) $, inputting $x(t)$ and $a(t)$. Upon receiving the query, the environment generates a resulting CGM value, and translates it into $ r(t) $ using the following mapping:
        \begin{align}
            f(x) = \begin{cases} 
                  0,                 & x   \leq 80 \ \ \text{(hypoglycemia)}\\
                  \frac{x - 80}{10}, & 80  \leq x \leq 90 \\
                  1,                 & 90  \leq x \leq 130 \\
                  \frac{180-x}{50},  & 130 \leq x \leq 180 \\
                  0,                 & 180 \leq x \ \ \text{(hyperglycemia)}
            \end{cases}
            \label{t1dm_rew_fn_eqn}
        \end{align}
    
    Note that we add zero-mean Gaussian noise with standard deviation of 5 to the resulting CGMs to introduce randomness to the rewards.

        \begin{figure}[h!]
    	\centering
    	\includegraphics[width=0.9\linewidth]{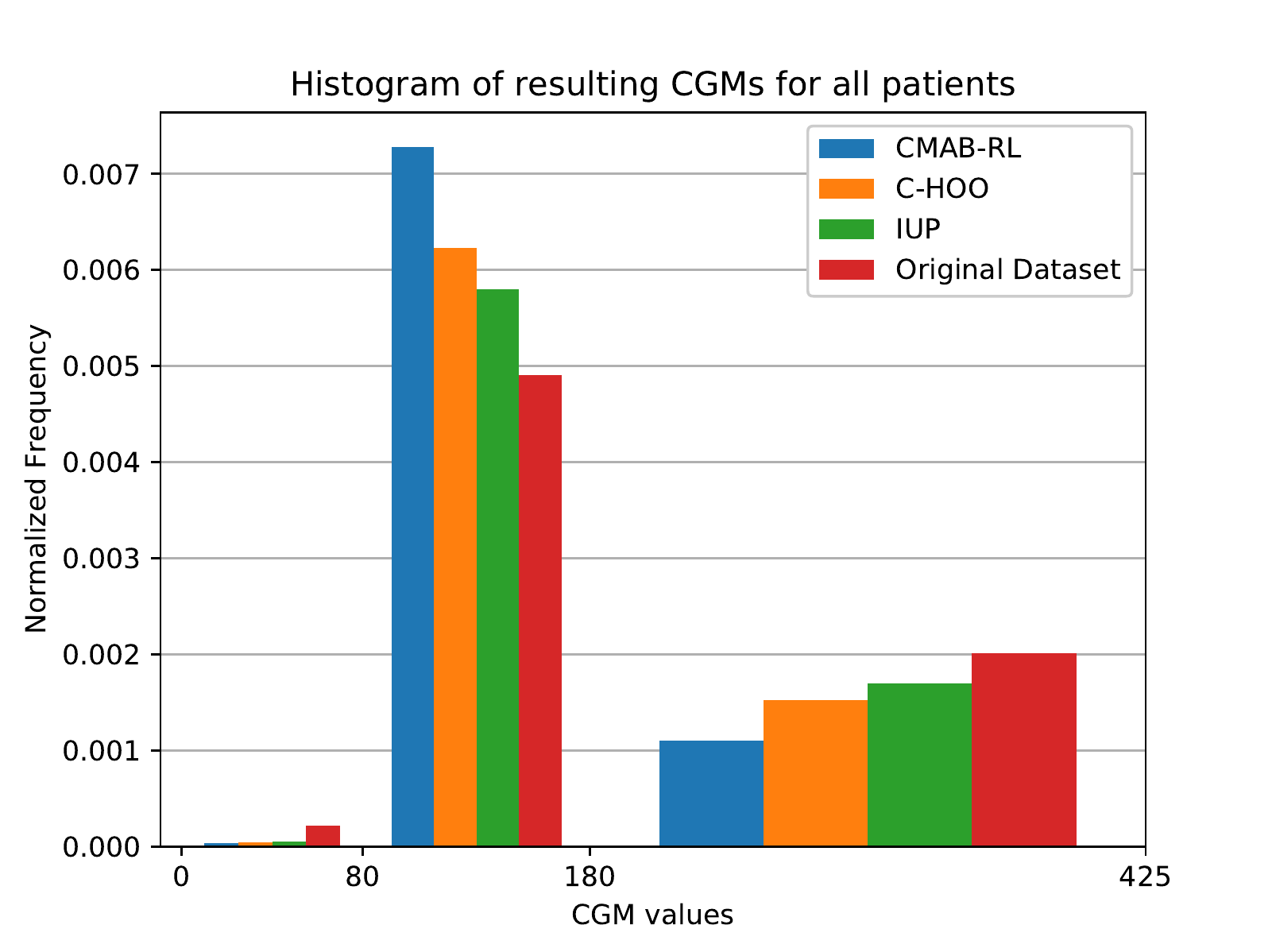}
    	\caption{Histograms of the resulting CGMs for all patients under different learning algorithms and the original dataset.}
    	\label{fig:result_hist}
    \end{figure}
    
   After training the Gradient Boosting regression model, we examine average impurity decrease for each input across all trees which are then normalized so that the sum of the average impurities for all inputs add up to 1. This examination shows that only the past CGM values before a bolus event yields a score higher than 0.5, while all the other variables yield scores lower than 0.1. This result is consistent with other works that study this dataset in the setting of forecasting, including \cite{zhu18deept1dm} and \cite{midroni18xgboostt1dm}. Therefore, it can be argued that past CGM values are the most relevant in the set of available features. In light of this information, we set $ \underline{d}_x = 1 $ during the experiment and fix the horizon to be $ T = 10^5 $. The confidence term multipliers in this experiment are $0.001$ for CMAB-RL, $0.05$ for IUP and $0.1$ for C-HOO.

    The histograms of resulting CGMs of all learning algorithms and the original dataset are given in Fig. \ref{fig:result_hist}. These are normalized such that the area under individual histograms sum up to $1$, so that the difference between the glucose control in the original dataset and that of the learning algorithms can be observed better. It is observed that in general all learning algorithms provide better glucose management than the one in the original dataset. In addition, Table \ref{hist_table}, represents the percentage of samples for which the resulting CGMs represent hypoglycemia or hyperglycemia, or are in the desired range. It is seen that for each patient, CMAB-RL has the highest percentage of samples between the desired range of $80$ to $180$ mg/dL. Moreover, CMAB-RL also has the lowest density in the regions that correspond to hypoglycemia and hyperglycemia, except for patient 588, for which the original dataset has no hypoglycemic CGMs.

\section{Conclusion}\label{conclusion}
In this work, we considered a CMAB problem with high-dimensional context and arm sets, and motivated by real-world applications, assumed that the reward only depends on a few relevant dimensions of the context and the arm sets. For this problem, we proposed an online learning algorithm, called CMAB-RL, which learns the relevant context and arm dimensions simultaneously, thereby achieving a regret bound that only depends on the maximum number of relevant dimensions given that this number is known by the learner. Our regret analysis does not require any stochastic assumptions on the context arrivals, and CMAB-RL is shown to beat other contextual MAB algorithms that do not exploit the relevance in both synthetic and real-world datasets.

\bibliographystyle{IEEEtran}
\bibliography{whole}

\newpage

\appendices

\section{Tables of Notation}

\begin{table}[h!]
	\caption{Notation related to problem formulation}
	\label{table:notation1}
	\small
	\centering
	\begin{tabular}{p{1cm} p{7cm} }
		\toprule
		Notation & Description \\
		\midrule
		$|{\cal S}|$ & Cardinality of a set ${\cal S}$ \\
		$\wp ({\cal S})$ & The power set of set ${\cal S}$ \\ 
		${\cal X}$   & Context set \\
		$d_x$      & Number of context dimensions \\	
		${\cal D}_x$ & Set of context dimensions: ${\cal D}_x \coloneqq \{1, \ldots, d_x\}$ \\
		${\cal V}^l_x$ & $\{ \bs{v} \in \wp( {\cal D}_x ): |\bs{v}| = l \}$ \\
		$ {\cal V}^{l}_{x} (\bs{v}) $ & $\{ \bs{w} \in {\cal V}^l_x : \bs{v} \subseteq \bs{w} \}$ for $\bs{v} \subseteq {\cal D}_x$  \\
		${\cal X}_{\bs{z}}$ & Subset of ${\cal X}$ that contains the values of context dimensions in $\bs{z} \subseteq {\cal D}_x$ \\
		$x_{\bs{z}}$ & $|\bs{z}|$-tuple subcontext whose elements are elements of $x$ that correspond to the context dimensions in $\bs{z}$ \\ 
		$x(t)$   & Context in round $t$ \\
		${\cal A}$ & Arm set  \\
		$d_a$ & Number of arm dimensions \\
		${\cal D}_a$ & Set of arm dimensions: ${\cal D}_a \coloneqq \{1, \ldots, d_a\}$ \\
		${\cal V}^l_a$ & $\{ \bs{v} \in \wp( {\cal D}_a ): |\bs{v}| = l \}$ \\
		${\cal A}_{\bs{z}}$ & Subset of ${\cal A}$ that contains the values of arm dimensions in $\bs{z} \subseteq {\cal D}_a$ \\
		$a_{\bs{z}}$ & $|\bs{z}|$-tuple subarm whose elements are elements of $a$ that correspond to the arm dimensions in $\bs{z}$ \\ 
	    $a^*(x)$ & Optimal arm for context $x$ \\
		$a(t)$ & Arm chosen by the learner in round $t$  \\
		$\mu_a(x)$  & Expected reward of arm $a$ for context $x$ \\
		$r(t)$  & Reward obtained in round $t$ \\
		$\kappa(t)$  & Noise in the reward in round $t$  \\
		$\bs{c}$ & Subset of relevant arm dimensions: $\bs{c} \subseteq {\cal D}_a$ \\
		$\underline{d}_a$ & Number of relevant arm dimensions \\	
	    $\overline{d}_a$ & Known upper bound on the number of relevant arm dimensions \\		
		$\bs{c}_a$ & Subset of relevant context dimensions for arm $a$: $\bs{c}_a \subseteq {\cal D}_x$ \\
		$\underline{d}_x$ & Maximum number of relevant context dimensions \\
		$\overline{d}_x$ & Known upper bound on the maximum number of relevant context dimensions \\
		$\text{Reg}(t) $ & Cumulative regret until round $t$
	\end{tabular}
	\end{table}		

\begin{table}[h!]
	\caption{Notation related to regret analysis}
	\label{table:notation3}
	\small
	\centering
	\begin{tabular}{p{1cm} p{7cm} }
		\toprule
		Notation & Description \\
		\midrule
		$\text{UC} $ & Unconfident event \\
		${\cal V}_x (\bs{v},d')$ & Set of $d'$-tuples of context dimensions whose elements are from the set ${\cal D}_x \setminus \bs{v}$ for any  $\bs{v} \in {\cal V}_x^{\overline{d}_x }$ and $d' \leq d_x - \overline{d}_x$, $d' \in \mathbb{Z}^+$ \\
		$\tau_{{p}_{\bs{w}}}(t)$ & Round in which a context arrives to ${p}_{\bs{w}}$ for the $t$th time \\
		$\tau_{y, {p}_{\bs{w}}}(t)$ & Round in which a context arrives to ${p}_{\bs{w}}$, arm $y$ is selected and $\bs{w} = \bs{w}_y(t) $ for the $t$th time \\
		$\tilde{x}_{{p}_{\bs{w}}}(t)$ & $x(\tau_{{p}_{\bs{w}}}(t))$ \\
		$\tilde{N}_{y,{p}_{\bs{w}}}(t)$ &  $N_{y,{p}_{\bs{w}}}(\tau_{{p}_{\bs{w}}}(t))$\\
		$\tilde{\mu}_{y,{p}_{\bs{w}}}(t)$ & $\hat{\mu}_{y,{p}_{\bs{w}}}(\tau_{{p}_{\bs{w}}}(t))$  \\
		$\tilde{u}_{y,{p}_{\bs{w}}}(t)$ & $u_{y,{p}_{\bs{w}}}(\tau_{{p}_{\bs{w}}}(t))$ \\
		$\tilde{y}_{{p}_{\bs{w}}}(t)$ & Arm chosen by the learner in round $\tau_{y, {p}_{\bs{w}}}(t)$\\
		$y^*(t)$ & Optimal arm in set ${\cal Y}$ for the context in round $t$ \\ 
		\bottomrule	
	\end{tabular}
\end{table}
		
	\begin{table}[h!]
			\caption{Notation related to CMAB-RL}
			\label{table:notation2}
			\small
			\centering
			\begin{tabular}{p{1cm} p{7cm} }
				\toprule
				Notation & Description \\
				\midrule		
		$m$ & Partition number \\ 
		${\cal P}_i $ & Uniform partition of the subcontext in dimension $i$: ${\cal P}_i := \{[0, \frac{1}{m}], ( \frac{1}{m}, \frac{2}{m}], \ldots, (\frac{m-1}{m},1] \}$ \\
		${\cal P}_{\bs{w}} $& $\prod_{i \in \bs{w}} {\cal P}_i $ \\
		${\cal I}_i $ & Uniform partition of the subarm in dimension $i$: ${\cal I}_i := \{[0, \frac{1}{m}], ( \frac{1}{m}, \frac{2}{m}], \ldots, (\frac{m-1}{m},1] \}$ \\
		${\cal I}_{\bs{v}} $& ${\cal I}_{\bs{v}} := \prod_{i \in \bs{v}} {\cal I}_i $ \\
		${\cal C}({\cal A})$ & $\cup_{ \bs{v} \in {\cal V}^{\overline{d}_a}_{a}} {\cal I}_{\bs{v}}$  \\
		$y$ & Index of the geometric centers of the elements of ${\cal C}({\cal A})$ \\		 
		${\cal Y}$ & Set of arms that correspond to geometric centers of the elements of  ${\cal C}({\cal A})$ \\
		$p_{\bs{w}}(t) $ & Element of ${\cal P}_{\bs{w}}$ that $x_{\bs{w}}(t)$ belongs to for $\bs{w} \in{\cal V}^{2 \overline{d}_x}_{x}$\\
		$N_{y, \bs{w}}(t)$ & Counter that counts the number of times context was in $p_{\bs{w}}$ and arm $y$ was selected before round $t$ \\
		$\hat{\mu}_{y, \bs{w}}(t)$ & Sample mean of the rewards that is obtained from rounds prior to round $t$ in which context was in $p_{\bs{w}}$ and arm $y$ was selected  \\
		$u_{y, \bs{w}} (t)$ & Uncertainty term for arm $y \in {\cal Y} $ and $\bs{w} \in{\cal V}^{2 \overline{d}_x}_{x}$ in round $t$ \\
		${\cal R}_y(t) $ & Set of candidate relevant tuples of context dimensions for $y \in {\cal Y}$ in round $t$ \\
		$\hat{\sigma}^2_{y, \bs{v}} $& Variance estimate of $\bs{v} \in {\cal R}_y(t)$ and $y \in {\cal Y}$ in round $t$ \\
		$\hat{\bs{c}}_y $ & Tuple of estimated relevant context dimensions for arm $y \in {\cal Y}$ \\
		$\hat{\mu}^{\bs{v}}_{y}(t)$ & Sample mean reward of arm $y \in {\cal Y}$ for the tuple of contex dimensions $\bs{v} \in {\cal V}_x^{\overline{d}_x} $ in round $t$\\
	\end{tabular}
\end{table}

\newpage

\section{Complexity of CMAB-RL}
\label{app:complexity} 

CMAB-RL stores sample mean reward estimates and counters for all $ y \in {\cal Y} $, $\bs{w} \in {\cal V}^{2\overline{d}_x}_x$ and $p_{\bs{w}} \in {\cal P}_{\bs{w}}$. Thus, the memory requirement is $ O\left(\genfrac(){0pt}{1}{d_a}{\overline{d}_a} \genfrac(){0pt}{1}{d_x}{2\overline{d}_x} m^{2\overline{d}_x + \overline{d}_a}\right) $. For $m = \lceil T^{1/(2+2\overline{d}_x + \overline{d}_a)} \rceil$ as given in Theorem 1, the memory complexity in time becomes $O( T^{(2\overline{d}_x + \overline{d}_a)/(2+2\overline{d}_x + \overline{d}_a)} )$, which is sublinear. Note that CMAB-RL can be implemented in a more efficient way by only creating and storing sample mean reward estimates and counters for sets in the partition to which a context had arrived in the past. 

Next, we investigate the computational complexity of CMAB-RL during run-time. In round $t$, finding $p_{\bs{w}}(t) \in {\cal P}_{\bs{w}}$ for all $\bs{w} \in {\cal V}^{2\overline{d}_x}_x$, requires $O \left(d_x + \genfrac(){0pt}{1}{d_x}{2\overline{d}_x}\right) $ computations. Construction of set $ {\cal R}_y(t) $ for all $ y \in {\cal Y} $, calculation of $\hat{\sigma}^2_{y, \bs{v}} (t)$ for all $ y \in {\cal Y} $, $\bs{v} \in {\cal R}_y(t)$ and determination of $\hat{\bs{c}}_y (t)$ for all $ y \in {\cal Y} $ all together require $ O\left(\genfrac(){0pt}{1}{d_a}{\overline{d}_a} m^{\overline{d}_a}  \genfrac(){0pt}{1}{d_x}{\overline{d}_x} {\genfrac(){0pt}{1}{d_x - \overline{d}_x}{\overline{d}_x}}^{2} \right)$ operations. 
In order to estimate $\hat{\mu}^{\hat{\bs{c}}_y(t)}_{y}(t)$ for all $ y \in {\cal Y} $, we need $O\left(\genfrac(){0pt}{1}{d_a}{\overline{d}_a} m^{\overline{d}_a} {\genfrac(){0pt}{1}{d_x - \overline{d}_x}{\overline{d}_x}} \right)$ operations. Determination of $\bs{w}_y(t)$ for all $ y \in {\cal Y} $ requires $ O\left(\genfrac(){0pt}{1}{d_a}{\overline{d}_a} m^{\overline{d}_a}  \genfrac(){0pt}{1}{d_x}{2\overline{d}_x} \right)$ comparisons. Finally, determining $y(t)$ requires $ O\left(\genfrac(){0pt}{1}{d_a}{\overline{d}_a} m^{\overline{d}_a} \right)$ comparisons. Hence, the overall computational complexity of CMAB-RL in round $t$ is 
\begin{align*}
O\left(\genfrac(){0pt}{1}{d_a}{\overline{d}_a} m^{\overline{d}_a}  \genfrac(){0pt}{1}{d_x}{\overline{d}_x} {\genfrac(){0pt}{1}{d_x - \overline{d}_x}{\overline{d}_x}}^{2} \right) .
\end{align*}
For $m = \lceil T^{1/(2+2\overline{d}_x + \overline{d}_a)} \rceil$ as given in Theorem 1, per-round computational complexity in time becomes $O( T^{\overline{d}_a/(2+2\overline{d}_x + \overline{d}_a)} )$, which is sublinear. 

\section{Regret Analysis of Exp4 for a Special Case}
\label{app:exp4} 

In this section, we assume that the set of relevant context dimensions is the same for each arm and derive a regret bound for Exp4 [46] under this assumption. We define the experts of Exp4 in the following way. Recall the Generate procedure from Section IV. For each $\bs{v} \in  {\cal V}^{\overline{d}_x}_x$, we have $|{\cal Y}| =\genfrac(){0pt}{1}{d_a}{\overline{d}_a}  m^{\overline{d}_a}$ arms. Expert $(\bs{v},i)$ assumes that the tuple of relevant context dimensions is $\bs{v}$. It partitions ${\cal X}_{\bs{v}}$ into $m^{\overline{d}_x}$ sets denoted by ${\cal P}_{\bs{v}}$. Then, it assigns one action in ${\cal Y}$ to each set in ${\cal P}_{\bs{v}}$. The number of different experts that can be defined for $\bs{v}$ is
\begin{align*}
|{\cal Y}|^{|{\cal P}_{\bs{v}}|} = \left( {d_a \choose \overline{d}_a} m^{\overline{d}_a} \right)^{ m^{\overline{d}_x}  } . 
\end{align*}
Thus, the total number of experts is
\begin{align*}
N = {d_x \choose \overline{d}_x}  |{\cal Y}|^{|{\cal P}_{\bs{v}}|} = {d_x \choose \overline{d}_x}  \left( {d_a \choose \overline{d}_a} m^{\overline{d}_a} \right)^{ m^{\overline{d}_x}  } . 
\end{align*}
The regret of the best expert in the pool of experts defined above with respect to the optimal arm is proportional to $T L (\sqrt{\overline{d}_a}/m +\sqrt{\overline{d}_x}/m )$ due to discretization. If we use Exp4, then its regret with respect to the best expert in the pool of experts defined above is 
\begin{align*}
&O \left( \sqrt{T |{\cal Y}| \log N} \right)  \\
&=  O \left( \sqrt{T {d_a \choose \overline{d}_a} m^{\overline{d}_a} \left(\log {d_x \choose \overline{d}_x}  + m^{\overline{d}_x} \log \left( {d_a \choose \overline{d}_a} m^{\overline{d}_a} \right)  \right) } \right)  \\
&= O \left( \sqrt{T {d_a \choose \overline{d}_a} m^{\overline{d}_a + \overline{d}_x}   \log  \left( {d_a \choose \overline{d}_a} m^{\overline{d}_a} \right) } \right) . 
\end{align*}
To balance the regret due to discretization and the regret due to Exp4 we set $m = \lceil T^{1/(\overline{d}_x + \overline{d}_a +2)} \rceil$, which results in total regret $\tilde{O}(T^{1 - 1/ (\overline{d}_x + \overline{d}_a +2)})$.

\section{Concentration Inequality [28]}
\label{app:concentration} 

Consider an arm $y$, a tuple $\bs{w}$, and a set $p_{\bs{w}}$ in the partition ${\cal P}_{\bs{w}}$ for which the rewards are generated by a process $\{ {R}_{y,{p}_{\bs{w}}}(t) \}_{t=1}^{ N_{y,{p}_{\bs{w}}}(T) }$ with $ {\mu}_{y,{p}_{\bs{w}}} = \mr{E} [{R}_{y,{p}_{\bs{w}}}(t)]$, where the noise ${R}_{y,{p}_{\bs{w}}}(t) - {\mu}_{y,{p}_{\bs{w}}}$ is conditionally 1-sub-Gaussian. Let $N_{y,{p}_{\bs{w}}}(T)\geq 1$ denote the number of times $y$ is selected by the end of time $T$. 
Let $\hat{\mu}_{y,{p}_{\bs{w}}}(T) = \sum_{t=1}^T \mr{I} (y(t) =y ) {R}_{y,{p}_{\bs{w}}}(t) / N_{y,{p}_{\bs{w}}}(T)$.
For any $\delta > 0$ with probability at least $1-\delta$ we have
\begin{align*}
& \big| \hat{\mu}_{y,{p}_{\bs{w}}}(T)  - {\mu}_{y,{p}_{\bs{w}}} \big|
\\ &\leq \sqrt{  \frac{2}{ N_{y,{p}_{\bs{w}}}(T) } 
	\left(1 + 2 \log \left(  \frac{ (1 + N_{y,{p}_{\bs{w}}}(T)  )^{1/2} } {\delta}    \right)  
	\right)  }  ~~ \forall T \in \mathbb{N}.   \notag
\end{align*}

\end{document}